\pdfoutput=1
\documentclass{article}

     \usepackage[preprint,nonatbib]{neurips_2020}

\usepackage[utf8]{inputenc} 
\usepackage[T1]{fontenc}    
\usepackage{url}            
\usepackage{booktabs}       
\usepackage{amsfonts}       
\usepackage{nicefrac}       
\usepackage{microtype}      

\usepackage{wrapfig}

\usepackage{xcolor} 

\definecolor{mydarkblue}{rgb}{0,0.08,0.45}

\usepackage[colorlinks=true,
    linkcolor=mydarkblue,
    citecolor=mydarkblue,
    filecolor=mydarkblue,
    urlcolor=mydarkblue]{hyperref} 

\usepackage{amsmath,amssymb,amsthm,amsfonts,latexsym}    \usepackage{mathtools}    
\usepackage{graphicx}
\usepackage{bm}
\usepackage{subfigure}
\usepackage{color}
\usepackage{appendix}
\usepackage{enumitem}
\usepackage{environ}         
\usepackage{etoolbox}        
\usepackage{setspace}
\usepackage{algpseudocode}
\usepackage{algorithm}

\usepackage[capitalise]{cleveref}

\newcommand{\x}{\vw}
\newcommand{\xk}{\vw_{k}}

\newcommand{\xkk}{\vw_{k+1}}
\newcommand{\xopt}{\vw^{*}}

\newcommand{\grad}[1]{\nabla f(#1)}

\newcommand{\dpr}[2]{\left< #1,#2\right>}

\newcommand{\norm}[1]{\left\|#1\right\|}
\newcommand{\normsq}[1]{\left\|#1\right\|^{2}}

\newcommand{\E}{\mathbb{E}}
\newcommand{\inv}[1]{#1^{-1}}
\newcommand{\covar}{\mathbf{\Sigma}}

\newcommand{\etak}{\eta_{k}}

\newcommand{\indnorm}[2]{\left\|#1\right\|_{#2}}
\newcommand{\indnormsq}[2]{\left\|#1\right\|_{#2}^{2}}

\newcommand{\transpose}{^\mathsf{\scriptscriptstyle T}}
\newcommand{\prp}{_{\mathsf{\scriptscriptstyle \perp}}}
\newcommand{\parll}{_{\mathsf{\scriptscriptstyle \parallel}}}

\DeclareMathOperator*{\argmax}{arg\,max}
\DeclareMathOperator*{\argmin}{arg\,min}

\newtheorem{lemma}{Lemma}

\newtheorem{example}{Example}
\newtheorem{proposition}{Proposition}

\newcommand{\pco}{\mxp^{-1}}
\newcommand{\wn}{\vw_0}
\newcommand{\winf}{\vw_{\infty}}
\newcommand{\wstar}{\vw^*}
\newcommand{\ex}[2][]{\mathbb{E}_{#1}[#2]}

\newcommand{\ma}{\mathbf{A}}
\newcommand{\mb}{\mathbf{B}}
\newcommand{\mc}{\mathbf{C}}
\newcommand{\mf}{\mathbf{F}}
\newcommand{\mg}{\mathbf{G}}
\newcommand{\mh}{\mathbf{H}}
\newcommand{\mi}{\mathbf{I}}
\newcommand{\mk}{\mathbf{K}}
\newcommand{\mm}{\mathbf{M}}
\newcommand{\mxp}{\mathbf{P}}
\newcommand{\mq}{\mathbf{Q}}
\newcommand{\ms}{\mathbf{S}}
\newcommand{\mv}{\mathbf{V}}
\newcommand{\mmu}{\mathbf{U}}
\newcommand{\mw}{\mathbf{W}}
\newcommand{\mx}{\mathbf{X}}
\newcommand{\mz}{\mathbf{Z}}

\newcommand{\va}{\mathbf{a}}
\newcommand{\vs}{\mathbf{s}}
\newcommand{\vw}{\mathbf{w}}
\newcommand{\vws}{\mathbf{w}^{*}}
\newcommand{\vx}{\mathbf{x}}
\newcommand{\vy}{\mathbf{y}}
\newcommand{\vz}{\mathbf{z}}

\newcommand{\vep}{\mathbf{\epsilon}}

\newcommand{\wopt}{\vw_{\text{opt}}}
\newcommand{\wmn}{\vw_{\text{mn}}}

\newcommand{\appendixTitle}{%
\vbox{
    \centering
	\hrule height 4pt
	\vskip 0.2in
	{\LARGE \bf Supplementary material}
	\vskip .5em
	{\large\bf  To Each Optimizer a Norm, To Each Norm its Generalization \vskip -\parskip \vskip 0.09in}
	\vskip 0.2in
	\hrule height 1pt 
}
}

\title{To Each Optimizer a Norm, \\ To Each Norm its Generalization}
\author{%
  Sharan Vaswani \\
  Mila, Universit\'e de Montr\'eal \\
   \And
   Reza Babanezhad\\
   SAIT AI Lab, Montreal \\
  \And
   Jose Gallego \\
   Mila, Universit\'e de Montr\'eal \\
   \AND
   Aaron Mishkin \\
   University of British Columbia \\
   \And
  Simon Lacoste-Julien$^*$ \\
  Mila, Universit\'e de Montr\'eal \\
  \And
 Nicolas Le Roux \thanks{Canada CIFAR AI Chair.} \\
Mila, 
Google Research - Brain Team
}

\begin{document}

\maketitle

\begin{abstract}
We study the implicit regularization of optimization methods for linear models interpolating the training data in the under-parametrized and over-parametrized regimes. Since it is difficult to determine whether an optimizer converges to solutions that minimize a known norm, we flip the problem and investigate what is the corresponding norm minimized by an interpolating solution. Using this reasoning, we prove that for over-parameterized linear regression, projections onto linear spans can be used to move between different interpolating solutions. For under-parameterized linear classification, we prove that for any linear classifier separating the data, there exists a family of quadratic norms $\norm{\cdot}_\mxp$ such that the classifier's direction is the same as that of the maximum $\mxp$-margin solution. For linear classification, we argue that analyzing convergence to the standard maximum $\ell_2$-margin is arbitrary and show that minimizing the norm induced by the data results in better generalization. Furthermore, for over-parameterized linear classification, projections onto the data-span enable us to use techniques from the under-parameterized setting. On the empirical side, we propose techniques to bias optimizers towards better generalizing solutions, improving their test performance. We validate our theoretical results via synthetic experiments, and use the neural tangent kernel to handle non-linear models.
\end{abstract}

\section{Introduction}
\label{sec:introduction}
Modern machine learning has seen the rise of large over-parameterized models such as deep neural networks~\cite{goodfellow2016deep}. These models are highly expressive and are able to fit or \emph{interpolate} all the training data~\cite{zhang2016understanding,belkin2018understand, belkin2019reconciling}. Since the number of parameters is much larger than the size of the training dataset, there are infinitely many solutions that can fit the data. These solutions can have vastly different generalization performance and the optimization method employed to minimize the training loss also influences the test performance~\cite{wilson2017marginal, keskar2017improving,qian2019implicit, gunasekar2018characterizing, arora2019implicit}. This is in contrast to classical regularized, under-parameterized models where there is a unique solution and the optimization method is solely responsible for converging to this solution at an appropriate rate. 

A recent line of work~\cite{gunasekar2017implicit,ji2018gradient, gunasekar2018implicit,nacson2018convergence,soudry2018implicit,gunasekar2018characterizing,nacson2019stochastic, arora2019implicit} studies the \emph{implicit regularization} of optimization methods in simplified settings. The implicit regularization of an optimizer biases it towards specific types of minimizers that are preferred amongst the infinite set of solutions. Two such simplified settings studied in this literature are over-parameterized linear regression with the squared loss~\cite{gunasekar2018characterizing} and linear classification on separable data using losses with an exponential tail~\cite{soudry2018implicit}. In each of these settings, recent works studies numerous optimization methods characterizing their implicit bias towards the minimum $\ell_2$ norm or maximum $\ell_2$ margin solutions respectively. 

A limitation of these works is that the implicit bias needs to be derived separately for each optimization method under different and often restrictive conditions. Even in simple scenarios like linear regression or classification, it is difficult to analyze the solutions of common optimization methods such as Adagrad~\cite{duchi2011adaptive}. Moreover, recent works~\cite{arora2019implicit, razin2020implicit} have shown that it is not possible to derive closed-form expressions for the implicit regularization in more challenging settings like matrix factorization. 

Rather than solving this difficult problem of determining whether an optimizer converges to solutions that minimize a known norm ($\ell_2$, Schatten norms); we flip the problem, and determine the corresponding norm minimized by an interpolating solution found by an optimizer. By reasoning along these lines, we develop techniques that can improve the generalization performance of common optimization methods. Moreover, we revisit the assumption that convergence to the minimum $\ell_2$ norm is desirable and find that regression and classification behave differently in this regard. 

\subsection{Background and Contributions}
\label{sec:contributions}
We consider linear models interpolating the training data in the under-parametrized and over-parametrized regimes. In particular, we study the implicit bias of optimizers in the linear regression (Section~\ref{sec:lin-reg}) and classification (Sections~\ref{sec:lin-class-under},\ref{sec:lin-class-over}) settings. 

\textbf{Over-parametrized linear regression}: In~\cite{gunasekar2018characterizing}, the authors study the implicit bias of gradient descent and its accelerated variants and show that it converges to the minimum $\ell_2$-norm solution. More generally, they characterize the implicit regularization for the steepest descent and mirror descent algorithms. In Section~\ref{sec:lin-reg}, we prove that every interpolating solution has a corresponding quadratic norm that it minimizes. We show that this result enables the use of projections onto linear spans to move between interpolating solutions. This further implies that for any interpolating solution found by an optimizer, a projection onto the data-span recovers the min-norm solution and can potentially improve the optimizer's generalization performance. We investigate, both theoretically and empirically, whether it is possible to find a norm that generalizes better than the $\ell_2$ norm.   

\textbf{Under-parameterized linear classification on separable data}: The implicit bias of gradient descent minimizing losses with an exponential tail has been studied in~\cite{soudry2018implicit, nacson2018convergence, nacson2019stochastic}. In these works, it was shown that the direction of the gradient descent (GD) solution converges to the max-margin solution at a $1/\log(T)$ rate, where $T$ is the number of GD iterations. In~\cite{gunasekar2018characterizing}, the authors also outline the implicit regularization properties of steepest descent, whereas the implicit bias of Adagrad is studied in~\cite{qian2019implicit}. Molitor et al.~\cite{molitor2020bias} studied the implicit regularization of gradient descent minimizing the non-smooth hinge loss. They consider a regularized problem with diminishing regularization, and characterize convergence to the max-margin solution. 

In Section~\ref{sec:lin-class-under}, we first show that for any linear classifier perfectly separating the data, there exists a family of corresponding quadratic norms $\norm{\cdot}_{\mxp}$ such that the classifier’s direction is the same as that of the maximum $\mxp$-margin solution. We then show, both theoretically and empirically, that the max-margin solution in the norm induced by the data results in better generalization than the well-studied $\ell_2$-norm. This result implies that it is important to consider the properties of the data when reasoning about the implicit bias and generalization performance of methods. In Sections~\ref{sec:lin-class-sq-hinge}-\ref{sec:lin-class-exp}, we analyze the implicit regularization of optimizers minimizing the squared-hinge loss and the exponential loss respectively. For both these losses, we propose heuristics that can bias optimization methods towards solutions with good generalization properties and can potentially improve their test performance. 

\textbf{Over-parameterized linear classification}: In Section~\ref{sec:lin-class-over}, we show that projections onto the data-span allow us to extend the results in Section~\ref{sec:lin-class-under} to the  over-parameterized setting. This enables us to use the techniques developed in the under-parameterized setting in this case as well. 

\textbf{Experimental evaluation}: In Section~\ref{sec:experiments}, we first validate our theoretical results on synthetic datasets. Furthermore, we use the neural tangent kernel~\cite{jacot2018neural} to go beyond the linear setting and demonstrate the effectiveness of our techniques. 

\section{Over-parametrized linear regression}
\label{sec:lin-reg}
We first consider over-parameterized linear regression with $n$ points $\{\vx_i,  y_i\}$ consisting of $d$-dimensional feature vectors $\vx_i$ s.t. $d > n$ and the corresponding labels/measurements $y_i$. We make the standard assumption~\cite{bartlett2019benign} that the true labels are corrupted with noise $\epsilon \sim \mathcal N (0,\sigma^2 \mi_d)$, implying that $\vy=\mx \vw^* + \vep$ where $\mx \in \mathbb R^{n\times d}$ is the matrix of features, $\vy \in \mathbb{R}^{n}$ is the vector of labels and $\vw^* \in \mathbb{R}^{d}$ is the ``true`` data-generating vector. We seek to minimize the squared loss, $\min_\vw f(\vw) \coloneqq \normsq{\mx \vw - \vy}_2$. 
If the matrix $\mx \mx\transpose$ is full rank, there are infinitely many solutions that can \emph{interpolate} or exactly fit the training dataset. Consequently, optimization methods achieving zero training loss converge to different solutions that can have vastly different generalization properties. We study the generalization performance of such interpolating solutions $\wopt$, that is $\mx \vw_{\text{opt}} = \vy$. 

\subsection{Convergence to the minimum norm solution}
\label{sec:lin-reg-min-norm}
In the over-parameterized regime, given that $\mx \mx\transpose$ is invertible, it is known that gradient descent (GD) initialized at the origin converges to the minimum $\ell_2$ norm solution $\vw_{mn}$, henceforth referred to as the \emph{min-norm solution}. 
\begin{align}
\wmn & = \argmin_{\vw} \frac{1}{2} \|\vw\|_{2}^2   \quad \text{s.t.} \quad \mx\vw=\vy \quad \implies \quad \wmn = \mx \transpose (\mx \mx\transpose)^{-1} \vy
\label{eq:min-norm} 
\end{align}    
The generalization properties of the min-norm solution have been thoroughly studied in the under-parameterized~\cite{bartlett2002rademacher} and more recently the over-parameterized interpolation regime~\cite{bartlett2019benign, muthukumar2020harmless,hastie2019surprises}. The min-norm solution is the unique point that interpolates the data and lies in the span of the feature vectors. This property has been used to analyze the implicit regularization of  common optimization methods~\cite{gunasekar2018characterizing}. Note that this property is unaffected by using stochastic gradients of a finite-sum, implying that mini-batch variants of optimizers have the same implicit regularization. In Appendix~\ref{app:newton}, we prove that iterates of the Newton method with Levenberg-Marquardt regularization~\cite{levenberg1944method,marquardt1963algorithm} lie in the span of the training data, implying that it converges to the min-norm solution. Similarly, we use this property to prove that full-matrix variants~\cite{agarwal2019efficient} (without the diagonal approximation) of adaptive gradient methods like Adagrad~\cite{duchi2011adaptive}, Adam~\cite{kingma2014adam} also converge to the min-norm solution (Appendix~\ref{app:adagrad}). These methods are more robust to the step-size than GD and have  better empirical convergence, implying that their full-matrix variants converge faster but generalize as well as GD. However, with the commonly used diagonal approximation, these methods result in iterates that do not lie in the data-span and are consequently not guaranteed to converge to the min-norm solution~\cite{gunasekar2018characterizing}. 

Preconditioned gradient descent (PGD) is the simplest method whose iterates do not lie in the span of the training data, but whose implicit regularization can be analyzed. The iterates corresponding to PGD with a constant positive-definite preconditioner $\mxp$ and constant step-size $\eta$ can be written as: $\vw_{k+1} = \vw_{k} - \eta \, \mxp \nabla f(\vw)$. In Appendix~\ref{app:pgd}, we show that PGD converges to the solution $\vw_\text{PGD} \coloneqq \lim_{k \to \infty} \vw_k = \vw_0 + \mxp \mx\transpose (\mx \, \mxp \, \mx\transpose)^{-1} [\vy - \mx \vw_0] $. Furthermore, when $\vw_0 = 0$, $\vw_\text{PGD}$ is the minimum $\pco$-norm solution, meaning that $\vw_\text{PGD} = \argmin \frac{1}{2} \|\vw\|_{\pco}^2 = \frac{1}{2} \vw\transpose \inv{\mxp} \vw$ such that $\mx \vw = \vy$. We note that this result can also be seen as a consequence of Theorem $1$ of~\cite{gunasekar2018characterizing}. In Lemma~\ref{lemma:unique-span-gen} of Appendix~\ref{apdx:prop-of-pgd}, we prove that the PGD solution is the unique point that interpolates the data and lies in the transformed data-span, that is $\vw_\text{PGD} \in \text{span}(\mxp \mx\transpose)$. These results imply that the properties of PGD are equivalent to that of GD, differing only in the norm. In fact, we now show that a similar equivalence holds for general optimization methods.
\begin{lemma}
For the solution $\vw_{opt}$ obtained by a generic optimizer on the squared loss, it is possible to construct a family of constant preconditioners $\mxp_{opt}$ such that PGD with a \emph{fixed} preconditioner in $\mxp_{opt}$ converges to $\vw_{opt}$. One such family can be given as:
\[
\mxp_{opt} = \left[\normsq{\vw_{opt}} I_d - \vw_{opt} \vw_{opt}\transpose + \frac{\mathbf{\nu} \mathbf{\nu}\transpose}{\langle \vw_{opt}, \mathbf{\nu} \rangle} \right]^{-1}
\]
Here $\mathbf{\nu} = \mx\transpose \mathbf{\alpha}$ where $\mathbf{\alpha}$ is a random vector such that $\langle \vw_{opt}, \mathbf{\nu} \rangle > 0$, ad $I_d$ is the $d\times d$ identity matrix. Note that for any $\mathbf{\nu}$, either $\mathbf{\nu}$ or $-\mathbf{\nu}$ satisfies this constraint. 
\label{lemma:opt-P-reduction}
\end{lemma}
The above lemma enables us to express the solution of an arbitrary optimization method in terms of a PGD solution for a family of constant preconditioners. It implies that an interpolating solution found by an optimization method is the unique minimum-norm solution in the $\mxp_{opt}$ norm. In the proposition below, we show that such an equivalence enables the use of projection operators to move between interpolating solutions.  
\begin{proposition}
Consider two optimization methods, their respective interpolating solutions $\vw_1$ and $\vw_2$ and equivalent preconditioners $\mxp_1$ and $\mxp_2$, constructed according to Lemma~\ref{lemma:opt-P-reduction}. Projecting $\vw_1$ onto the $\text{span}(\mxp_2 \, \mx\transpose)$ using the projection operator $\mathbf{\pi} = \mxp_2 \mx\transpose (\mx \mxp_2 \mx \transpose)^{-1} \mx$ recovers $\vw_2$. 
\label{prop:proj-span}
\end{proposition}
The above result implies that projecting any optimizer's interpolating solution onto the data-span recovers the min-norm solution and can potentially improve its generalization performance. From a practical standpoint, we note that in the linear regression setting, the cost of such a projection is equal to that for solving the normal equations. However, such projections could be used to improve the generalization of optimizers minimizing the squared loss for very wide neural networks by making use of the neural tangent kernel~\cite{jacot2018neural}. Proposition~\ref{prop:proj-span} can be generalized to an arbitrary iterate, in that the same projection operator $\mathbf{\pi}$ moves $\vw_1$ to the $\text{span}(\mxp_2 \, \mx\transpose)$ without changing its training loss.

Since different norms have different generalization properties, we attempt to find a norm that results in better generalization than the minimum $\ell_2$ norm solution. In Appendix~\ref{app:gen-bounds}, we generalize the excess risk bounds in ~\cite{bartlett2019benign} to analyze solutions found by PGD. In Appendix~\ref{app:proofs-lin-reg}, we optimize an upper bound on the risk w.r.t $\mxp$. Even assuming full knowledge of the covariance matrix of the data, we show that it is not possible to uniformly improve over the min-norm solution in the noiseless case ($\sigma = 0$). We can obtain a better upper bound on the risk in the noisy case, however, we empirically demonstrate that the looseness of this bound prevents obtaining better generalizing solutions in practice. 

\section{Under-parameterized linear classification}
\label{sec:lin-class-under}
In this section, we consider binary classification with a training dataset $\{\vx_i,  y_i\}_{i = 1}^{n}$ of $d$-dimensional feature vectors $\vx_i$ and labels $y_i \in \{-1,1\}$, with $d\leq n$. We seek to find a hyperplane $\vw^*$ that minimizes the $0$-$1$ loss, $\vw^* = \argmin_{\vw} \sum_{i = 1}^{n} \mathbb{I} \{y_i \dpr{\vw}{\vx_i} \geq 0\}$ where $\mathbb{I}$ is the indicator function equal to $1$ when true and $0$ otherwise\footnote{We only consider homogeneous linear classifiers without a bias term}. Unlike under-parameterized regression that has a unique minimizer, there can be infinitely many linear classifiers or hyperplanes that separate the data. We study the interpolation setting where the data is linearly separable by a non-zero margin, implying that there exist linear classifiers with \emph{zero training error} or zero $0$-$1$ loss. Similar to over-parameterized regression, optimization methods achieving zero training error are biased towards certain solutions and can converge to hyperplanes that have different generalization properties. 

For a general positive definite matrix $\mxp$, if the data is separable by a margin (in the $\mxp$ norm) equal to $\gamma$, the maximum $\mxp$-margin solution $\vw_{\text{mm}, \mxp}$ has the following equivalent forms\footnote{For notational convenience, from now on, we absorb the label $y_i$ into the feature $\vx_i$.}:
\begin{align}
\vw_{\text{mm}, \mxp} & \coloneqq \argmax_{\norm{\vw}_\mxp \leq 1/\gamma} \,\, \min_{i \in [n]}  \dpr{\vw}{\vx_i} = \argmin_{\vw} {\norm{\vw}_{\mxp^{-1}}} \quad \text{s.t, for all $i$,} \quad \dpr{\vw}{\vx_i} \geq \gamma 
\label{eq:max-margin}
\end{align}
When $\mxp = I_d$, the corresponding maximum margin solution $\vw_\text{mm}$ is the standard max $\ell_2$-margin solution, henceforth referred to as the \emph{max-margin} solution. In this case, the quantity $\max_{\norm{\vw}_2 \leq 1/\gamma} \min_{i \in [n]} \dpr{\vw}{\vx_i}$ is the $\ell_2$ margin and data points corresponding to the equality $\dpr{\vw_\text{mm}}{\vx_i} = \gamma$ are the support vectors for $\vw_\text{mm}$.  The max-margin solution is shown to have good generalization performance for under-parameterized models~\cite{kakade2009complexity} and more recently in the over-parameterized setting~\cite{chatterji2020finite}. We first show that the direction of any linear classifier separating the data is the same as that of a maximum $\mxp$-margin solution for an appropriately constructed $\mxp$. 
\begin{lemma}
For an interpolating linear classifier $\vw$, it is possible to construct a family of quadratic norms $\norm{\cdot}_\mxp$ such that the direction of the classifier is equivalent to the direction corresponding to the max $\mxp$-margin solution where
\begin{align*}
\mxp = \left[\normsq{\vw} I_d - \vw \vw\transpose + \nu \nu\transpose \right]^{-1} \quad \text{s.t.} \quad \langle \vw, \nu \rangle  = 1. 
\end{align*}
Here, $\nu = \mv \transpose \mathbf{\alpha}$ where $\mv \in \mathbb{R}^{\vert S \vert \times d}$ is the feature-matrix corresponding to the set $S$ of support vectors for $\vw$, $\mathbf{\alpha}$ is a random vector satisfying the above constraints. 
\label{lemma:exp-equivalent-precond}
\end{lemma}
The above equivalence can be used to get a handle on the generalization performance of $\vw$. In particular, we first show that the generalization performance of the maximum $\mxp$-margin solution depends on the induced norm it minimizes. We then investigate whether it is possible to construct norms that generalize better than the $\ell_2$ max-margin solution. 

From relation~\ref{eq:max-margin}, observe that the maximum $P$-margin solution minimizes $\norm{\cdot}_{\pco}$. Let us consider an equivalent hypothesis class that has a (small) bounded $\pco$ norm and is given by $\mathcal F(\mxp)= \{\vx \to y\vw\transpose \vx |  \frac{1}{2} \vw\transpose \inv{\mxp} \vw \leq E\}$. We measure the complexity of this hypothesis class in terms of its Rademacher complexity or VC-dimension and obtain bounds on its generalization performance~\cite{shalev2014understanding}. We estimate these complexity measures via empirical values that are concentrated around the corresponding true value with high probability~\cite{shawe2004kernel,shalev2014understanding}. For data separable by an $\ell_2$ margin $\gamma$, let $\zeta_i = (\gamma - y_i \vw^T\vx)_+$ be the error on training sample $i$. Then with probability $1-\delta$,  
\begin{equation}
    \label{eq:mm_upper_gen}
    P_{\mathcal D}(y\neq sign(\vw\transpose \vx) )\leq \frac{1}{n\gamma} \sum_{i=1}^{n} \zeta_i + \frac{1}{\gamma}\hat{ \mathcal R}(\mathcal F(\mxp)) + 3\sqrt{\frac{\ln(2/\delta)}{2n}},
\end{equation}
which is an upper bound for the generalization error. The term $\hat{\mathcal R}(\mathcal F(\mxp))$ denotes the empirical Rademacher complexity. Since the generalization error depends on the $\inv{\mxp}$ norm via $\hat{ \mathcal R}(\mathcal F(\mxp))$, we minimize this quantity w.r.t. $\mxp$ and obtain the following lemma (proved in Appendix~\ref{sec:RC_Precond}). 
\begin{lemma}
\label{lemma:RC_Precond}
The Rademacher complexity of the model family $\mathcal{F} = \{y\vw\transpose \vx | \frac{1}{2}\vw\transpose \inv{\mxp} \vw \leq E\}$ is upper bounded by 
\begin{equation}
    \label{eq:rc_pmm}
    \hat{\mathcal{R}}(\mathcal{F}) \leq \frac{2\sqrt{2E}}{n} \sqrt{tr(\mxp \hat{\Sigma})}
\end{equation}
where $\hat{\Sigma}$ is the scaled covariance matrix of the data i.e. $\hat{\Sigma} = \mx\transpose \mx $. By constraining $\mxp$ to be symmetric positive definite, $\mxp^* = \inv{\hat{\Sigma}}$ minimizes the Rademacher complexity. 
\end{lemma}

\begin{wrapfigure}{r}{0.3\textwidth}
 \vspace{-4.5ex}
  \begin{center}
    \includegraphics[trim=0 0 0 10, clip, width=0.3\textwidth]{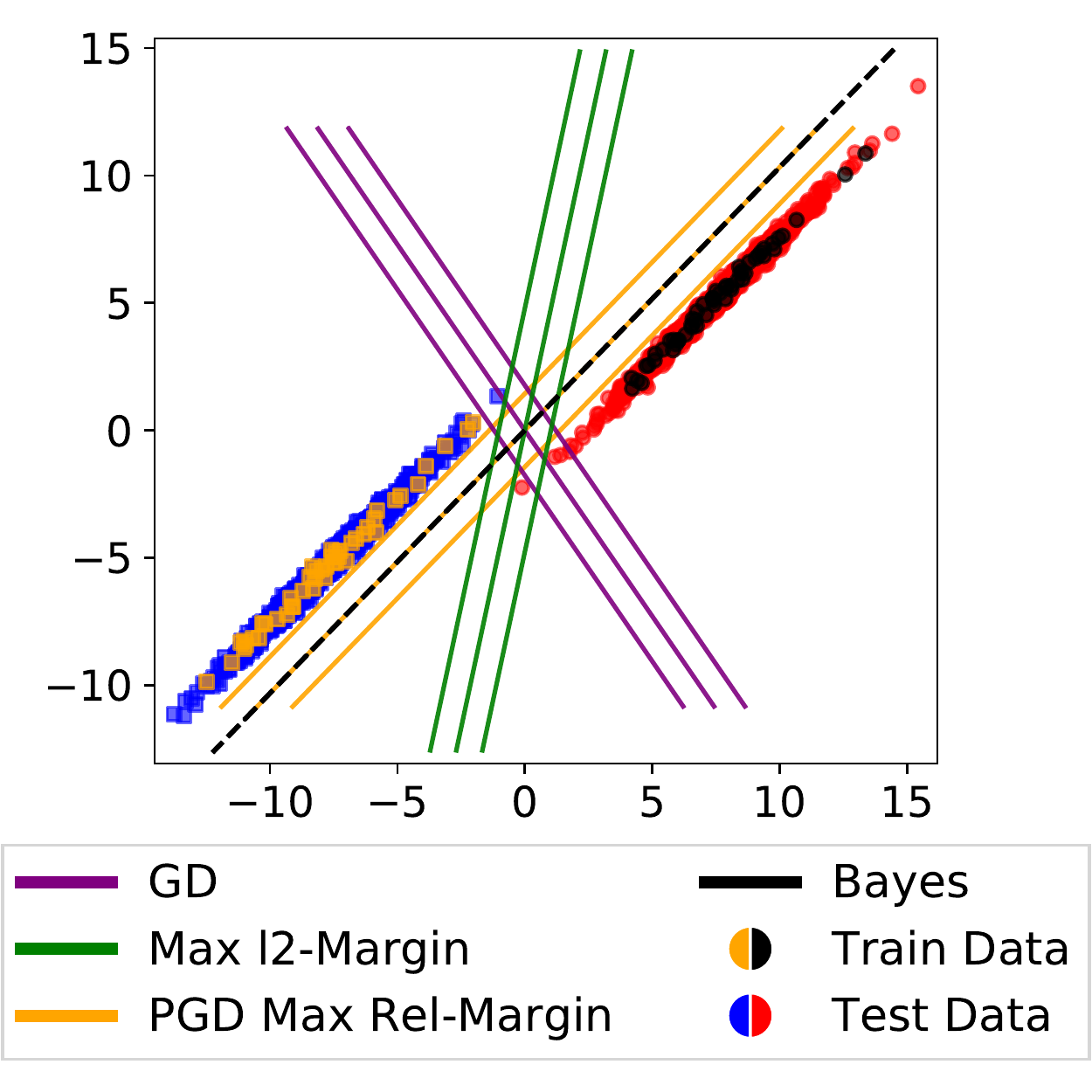}
  \end{center}
  \vspace{-3.5ex}
  \caption{Comparing solutions with maximum $\ell_2$ and relative margin performance on a synthetic mixture of Gaussians dataset.}
  \label{fig:sigma_prec_gd}
\end{wrapfigure}


If the two classes have different covariances, it can be easily shown that $\mxp^* = \inv{(\hat{\Sigma}_+ + \hat{\Sigma}_-)}$ where $\hat{\Sigma}_+$ and $\hat{\Sigma}_-$ is the scaled covariance matrix for the positive and negative classes respectively. Furthermore under a mild assumption, we show that $\mxp^*$ also minimizes an upper-bound on the VC-dimension. The above lemma shows that maximum margin solutions in the $\hat{\Sigma}$-norm can generalize better than the $\ell_2$ max-margin solution. Finally, we note that the margin in the $\hat{\Sigma}$-norm was defined as the \emph{relative margin} in~\cite{shivaswamy2010maximum} where the authors modified the standard SVM formulation to maximize the relative margin~\cite{shivaswamy2007ellipsoidal}. We present a simple example~\cite{shivaswamy2010maximum} (details in  Appendix~\ref{app:relative-margin}) to empirically validate the above result. In Figure~\ref{fig:sigma_prec_gd}, we show the effect of the proposed preconditioning for gradient descent: incorporating the covariance of the data maximizes the relative margin resulting in a solution which is better aligned with the Bayes optimal classifier. The above results show that measuring convergence w.r.t to the $\ell_2$ max-margin is arbitrary, and norms incorporating the structure of the data can generalize better. 

Next we consider whether optimization methods converge to solutions with good generalization properties. Since the $0$-$1$ loss is non-convex and difficult to minimize, we consider optimizers minimizing two surrogate loss functions - the squared-hinge loss and the exponential loss. In each of these cases, we develop techniques to bias the optimizers towards the maximum $\ell_2$ or $\hat{\Sigma}^{-1}$-margin solutions that have good generalization properties.  


\subsection{Squared hinge loss}
\label{sec:lin-class-sq-hinge}
The squared-hinge loss is a smooth, convex loss: $f(\vw) \coloneqq \frac{1}{n}\sum_{i = 1}^{n} \left(\max\{0,1 - \langle \vw, \vx_i \rangle\} \right)^{2}$. In this case, an interpolating solution achieves zero training loss. In Appendix~\ref{app:sq-hinge-counter}, we first show that gradient descent with an arbitrary initialization is not guaranteed to converge to the max-margin solution. Furthermore, we show that even with a zero initialization, GD is not guaranteed to converge to the max-margin solution for any constant step-size. To guarantee convergence to the max-margin solution, we assume knowledge of the true $\ell_2$-margin $\gamma$. Under this assumption, we analyze the convergence of \emph{projected GD} that projects the iterates onto the $\ell_2$ ball with radius $1/\gamma$ after each gradient step. The following proposition characterizes the rate of convergence of the empirical margin $\hat{\gamma}(\vw) \coloneqq \frac{\min_{i \in [n]} \langle \vx_i, \vw \rangle}{\norm{\vw}_2}$, which is a measure of convergence to the max-margin solution~\cite{soudry2018implicit}. 
\begin{proposition}
Assuming knowledge of the true margin $\gamma$, if $f$ is $L$-smooth, projected GD and projected GD with Nesterov acceleration results in the following convergence rate,
\begin{align*} 
\hat{\gamma}_\text{GD}(\x_T) \geq \gamma(1 - \frac{n\sqrt{L} \, \norm{\x_0 - \xopt} }{\sqrt{T}} ) \quad \text{;} \quad \hat{\gamma}_{\text{Acc-GD}}(\x_T) \geq \gamma (1 - \frac{n\sqrt{L} \, \norm{\x_0 - \xopt} }{T}).
\end{align*}
\label{thm:sqhinge-projgd-rate}
\end{proposition}
As $T \rightarrow \infty$, the empirical margin converges to the true margin. This rate of convergence to the max-margin solution is faster than the $O(1/\log(T))$ rate GD achieves when minimizing the exponential loss. The lower bound on max-margin convergence is $O(\sqrt{n/T})$~\cite{clarkson2012sublinear}, meaning that projected GD with Nesterov acceleration has a better dependence on $T$ at the cost of a worse dependence on $n$. 

Projection onto the $\ell_2$ ball results in the constrained optimization, $\min_{\vw} f(\vw)$ such that $\norm{\vw}_{2} = 1/\gamma$. This is a smooth, convex problem with a convex constraint set and can be solved by a generic optimization method. Furthermore, by choosing the radius of the ball to be the reciprocal of the margin, the constrained optimization problem has a unique minimizer equal to the max-margin solution. This results implies that (assuming knowledge of the true data margin) the projected variant of a generic optimization method is guaranteed to converge to the max-margin solution. If we assume knowledge of the margin in the $\mxp$ norm, the same result holds more generally. This implies that projections onto the $\indnorm{\cdot}{\hat{\Sigma}}$ ball of the appropriate radius can recover the maximum $\hat{\Sigma}^{-1}$-margin solution with better generalization properties. 

\subsection{Exponential loss}
\label{sec:lin-class-exp}
The exponential loss is a smooth, convex loss minimized at infinity for separable data and given by: $f(\x) \coloneqq \frac{1}{n} \sum_{i = 1}^{n} \exp{(-\dpr{\x}{\vx_i})}$. Previous works~\cite{nacson2018convergence, gunasekar2018characterizing, soudry2018implicit} show that (stochastic) gradient descent converges to the max-margin solution and the resulting empirical margin converges to the true margin at an $O(1/\log(T))$ rate. This result can be extended to general losses with an exponential tail, including the logistic loss~\cite{soudry2018implicit}. Similar to the regression setting, we first analyze the implicit regularization of PGD. In particular, we use the result in~\cite{soudry2018implicit} and state the following lemma (for completeness, we provide the proof in Appendix~\ref{app:proofs-lin-class}),  
\begin{lemma}
The empirical margin for PGD with preconditioner $P$ and constant step-size $\eta < 1/f(\x_0)$ satisfies:
\begin{align*}
\frac{\min_{j} \langle \xkk, \vx_j \rangle}{\indnorm{\xkk}{\mxp^{-1}}} & \geq \gamma - \left[\frac{\gamma \indnorm{\x_0}{\pco} + \eta f(\x_0) + \log(f(\x_0))}{\indnorm{\x_0}{\pco} + \log \left( \eta \gamma^2 (k+1) \right)} \right]
\end{align*}
The empirical margin converges to the true margin $\gamma$ (in the $\pco$-norm) at an $O(1/\log(k))$ rate.
\label{lemma:exp-loss-pgd}
\end{lemma}
Note that the empirical margin is measured in the $\pco$ norm and PGD with a constant step-size\footnote{In Appendix~\ref{app:proofs-lin-class}, we show the same result for PGD with an Armijo line-search with bounded step-sizes.} converges to the maximum $\mxp$-margin solution. For a general optimizer, consider an iterate $\vw_T$ obtained after $T$ iterations. If the classifier corresponding to $\vw_T$ correctly separates the data, we can find an equivalent $\mxp$ (in the sense of Lemma~\ref{lemma:exp-equivalent-precond}). Furthermore, using the above result, this direction corresponds to the direction of a PGD solution with a preconditioner $\mxp$. This result implies that, as in linear regression, an interpolating solution is equivalent (in direction) to a PGD solution. 

However, unlike regression, it is not possible to provably recover the maximum $\ell_2$-margin solution from a general interpolating solution found by an optimizer. Instead, we consider an empirical strategy that \emph{switches} from the original optimization method and runs ``some'' iterations of PGD to bias the resulting solution towards the corresponding max-margin direction. For example, switching to GD will result in a bias to the max-$\ell_2$-margin, and we can obtain better generalization by using $\mxp = \Sigma^{-1}$ to bias the optimizer to the maximum relative-margin solution. Since Lemma~\ref{lemma:exp-loss-pgd} holds for an arbitrary initialization, we invoke it with $\vw_0$ equal to the iterate obtained by the original optimization method. Since convergence to the max-margin solution depends on the loss $f(\vw_0)$, if the original optimization method is effective in minimizing the loss, the resulting $f(\vw_0)$ is small and $\norm{\vw_0}$ is large, making it possible to approach the corresponding max-margin solution in fewer iterations. In Section~\ref{sec:experiments}, we empirically demonstrate that switching to GD for only a few iterations can significantly improve the generalization performance of the original optimizer. We note that such a strategy of switching to GD (from Adam) has been explored in~\cite{keskar2017improving} in the context of deep networks and found to improve the generalization performance of Adam. Our reasoning using Lemma~\ref{lemma:exp-loss-pgd} gives further intuition on why such a strategy is reasonable. 




\section{Over-parameterized linear classification}
\label{sec:lin-class-over}
In this section, we consider over-parameterized linear classification where $d > n$ and investigate whether common optimization methods converge to the max-margin solution. Note that the max-margin solution $\vw_{mm} = \sum_{i = 1}^{n} \alpha_i \vx_i$ (where $\alpha_i$ are the corresponding dual variables) lies in the span of the training data. This implies that similar to the over-parameterized linear regression setting, a method can converge to $\vw_{mm}$ only if its iterates either lie in the data span or are projected onto it. Similar to the under-parameterized setting, we first consider the convergence of zero-initialized gradient descent when minimizing the exponential and squared hinge losses. Regardless of the loss, the iterates of GD lie in the data span and the optimization happens in an $n$-dimensional subspace. This implies that when analyzing optimization methods whose iterates lie in the span, the over-parameterized and under-parameterized settings are equivalent and the ideas from Section~\ref{sec:lin-class-under} can be used to obtain convergence to the max-margin solution. 

Next, we consider PGD as a simple algorithm whose iterates do not lie in the data-span. Similar to Section~\ref{sec:lin-reg}, one can project the final solution onto the data-span, however, unlike linear regression, such a projection does not guarantee convergence to the max-margin solution. To guarantee such convergence, we propose to use a projected variant of PGD. Projected PGD projects the iterates after each PGD step and has the following update rule: $\vw_{k+1} = \pi[\vw_{k} - \eta \mxp \grad{\vw_k}]$ where $\pi = \mx \inv{\left(\mx \mx\transpose \right)} \mx\transpose$ is the projection operator. The projection ensures that each iterate of this method lies in the data span. In the lemma below (proved in Appendix~\ref{app:over-lin-class-equivalence}), we show that projected PGD can be thought of as PGD with an equivalent preconditioner that lies in the data span. 
\begin{lemma}
Projected PGD with preconditioner $\mxp$ and the projection operator $\pi = \mx \transpose (\mx \mx\transpose)^{-1} \mx$ is equivalent to PGD with a preconditioner $\mx\transpose \bar{\mxp} \mx$ where $\bar{\mxp} = (\mx \mx \transpose)^{-1} \mx \mxp \mx \transpose (\mx \mx \transpose)^{-1}$ that lies in the data span. 
\label{lemma:over-lin-class-equivalence}
\end{lemma}
Note that although $\mx\transpose \bar{\mxp} \mx$ is not full rank, it spans the data-span, and preserves the gradient directions in this subspace. This lemma entails that we can use the techniques developed in Section~\ref{sec:lin-class-under} in the over-parameterized case given that we project the iterates onto the data span after each update. 

For the squared-hinge loss and given knowledge of the true margin $\gamma$, this implies a projected variant of a generic optimization method that first projects onto the data-span and then onto the $\ell_2$ ball of radius $1/\gamma$ will converge to the max-margin solution. For the exponential loss, since Lemma~\ref{lemma:exp-loss-pgd} holds for PGD with a general preconditioner, the equivalence in Lemma~\ref{lemma:over-lin-class-equivalence} implies that it also holds for projected PGD in the over-parameterized setting. Similarly, this equivalence ensures that the result of Lemma~\ref{lemma:exp-equivalent-precond} also holds for projected PGD. From a practical perspective, we show in Section~\ref{sec:experiments} that projecting only the final solution onto the data-span and then switching to GD towards the end of the optimization biases the iterates towards the max-margin solution. As in the under-parameterized setting, we can obtain better generalization than the maximum $\ell_2$ margin solution by considering the maximum relative margin solution. We empirically verify these observations in the next section.

\section{Experimental Evaluation}
\label{sec:experiments}
We verify the theoretical results for both the regression and classification settings, and evaluate the proposed techniques for improving the test performance of optimizers.

\begin{figure}[!h]
\centering
\includegraphics[width=1\textwidth]{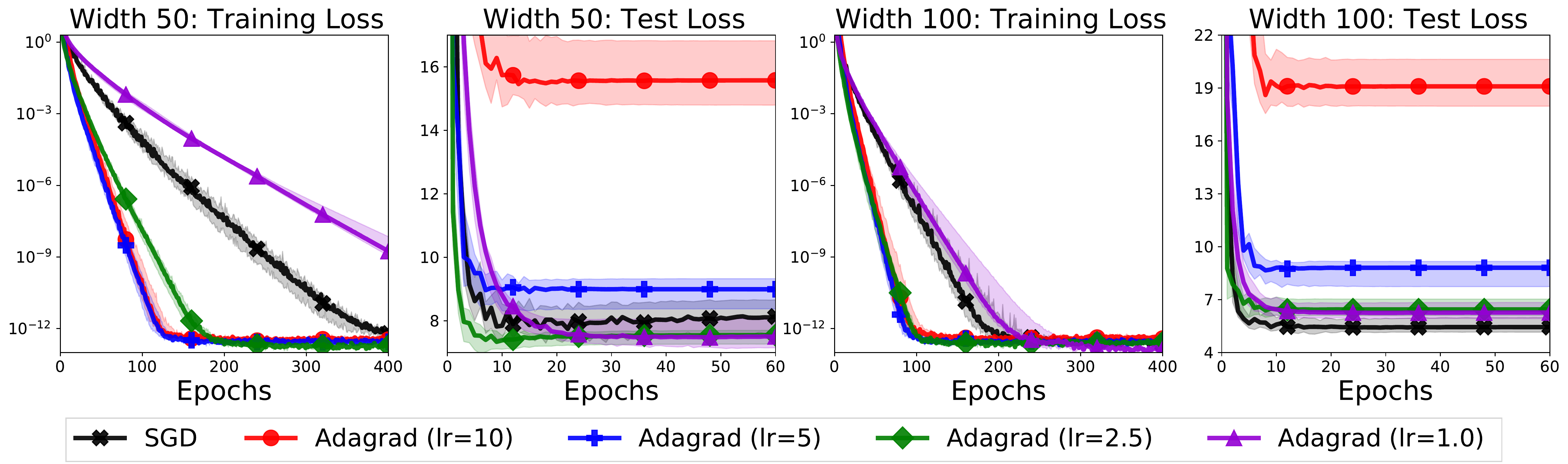}
\caption{Performance of SGD and Adagrad on a synthetic regression problem using the NTK of one-layer networks with 50 and 100 hidden units. Tuned SGD converges slowly to the min-norm solution. In contrast, the convergence of Adagrad is more robust to the step-size, however its generalization depends on the step-size. }
\label{fig:ntk_synthetic_regression}
\end{figure}

\begin{figure}[!h]
\centering
\includegraphics[width=1\textwidth]{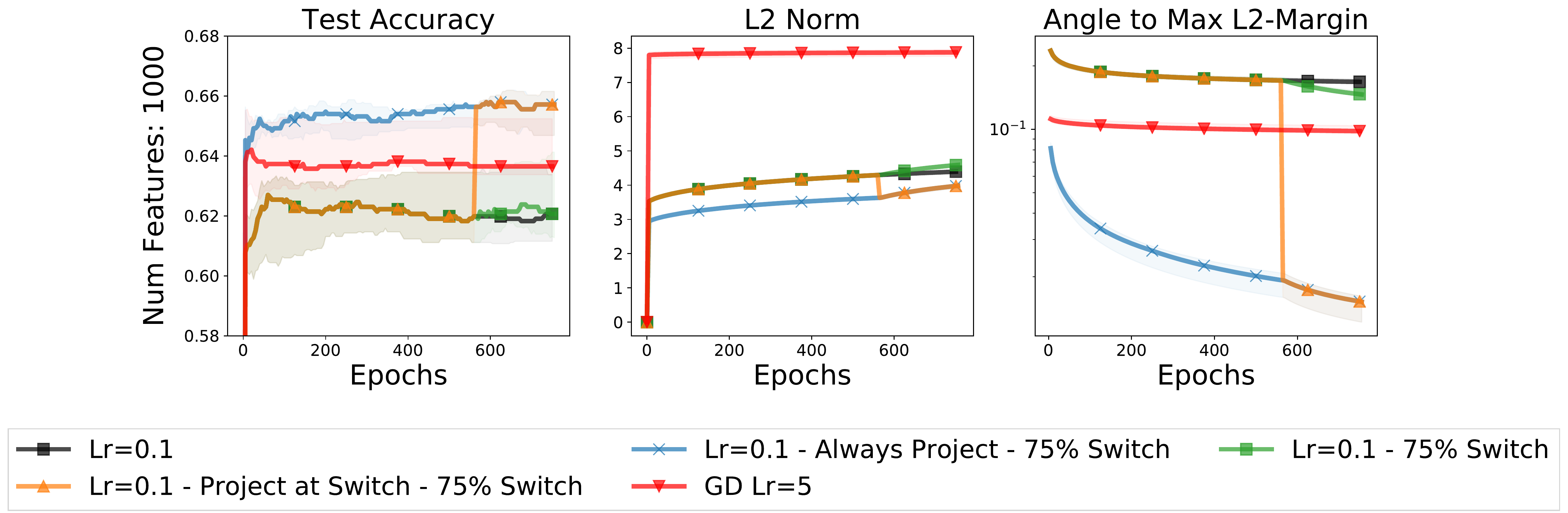}
\vspace{-3ex}
\caption{Performance of GD and Adagrad on a synthetic overparameterized classification problem with random Gaussian features. Projecting onto the data-span improves the test accuracy, while decreasing the solution's norm and angle to the max-margin solution. 
}
\label{fig:overp_classification}
\vspace{-3ex}
\end{figure}
 
\textbf{Regression}: We consider over-parameterized regression with one-layer neural networks. We generate data-points from a Gaussian distribution and use the resulting neural tangent kernel~\cite{jacot2018neural} (NTK) features as input to our problem. The targets are generated to ensure that the problem is realizable. For this problem, we study the generalization of SGD and Adagrad. From Figure~\ref{fig:ntk_synthetic_regression}, we observe that (i) Adagrad is robust to the choice of step-size and converges quickly, (ii) hand-tuned SGD converges slowly due to the problem's ill-conditioning, (iii) Adagrad's test performance highly depends on the step-size, and (iv) the min-norm solution found by SGD consistently generalizes well. We can project the Adagrad solution onto the data span as in Proposition~\ref{prop:proj-span} and recover the min-norm solution. This suggests we can benefit from the robustness of Adagrad, while also ensuring good generalization performance. 

In Appendix~\ref{app:additional_ntk_regression}, we present additional results on the \texttt{wine} and \texttt{mushroom} datasets~\cite{Dua:2019}. In Appendix~\ref{app:regression_verification}, we validate that full-matrix Adagrad and the Newton methods converge to the min-norm solution and verify the construction of the equivalent preconditioner in Lemma~\ref{lemma:opt-P-reduction}. In Appendix~\ref{app:best-p}, we investigate whether it is possible to obtain solutions that generalize better than the min-norm solution by minimizing upper bounds on the excess risk. 

\textbf{Classification}: We evaluate effectiveness of 
switching to GD (Section~\ref{sec:lin-class-under})
and projection (Section~\ref{sec:lin-class-over}) in improving the generalization for over-parameterized linear classification. In particular, we present the results for Adagrad and GD minimizing the logistic loss on a synthetic dataset with random Gaussian features. We ensure that both GD and Adagrad converge to solutions which interpolate the training data. We consider four variants: (i) standard Adagrad (black), (ii) switch to GD after 75\% of the iterations (green) (iii) project onto the data-span after every iteration (blue) and (iv) project the final solution onto the data-span (orange). For the latter two variants, we also switch to GD with a fixed step-size after 75\% of the iterations. From Figure~\ref{fig:overp_classification}, we observe that (i) switching to GD results in a small improvement in the test accuracy. (ii) projections after every iteration result in a smaller norm and angle to the max-margin solution as well as higher test accuracy, (iii) projecting only the final solution is sufficient to improve the test performance (from an accuracy of $64$\% to $66$\%) (iv) compared to projection, switching to GD has a small effect on the test accuracy. Our results indicate that being in the correct subspace can improve the generalization performance. 

In Appendix~\ref{app:overp_cls}, we presents additional results for this setting. In Appendix~\ref{app:equiv_prec_cls}, we verify the construction of the equivalent preconditioner in Lemma~\ref{lemma:exp-equivalent-precond}. Finally, in Appendix~\ref{app:sqh_ball_cls}, we validate that for the squared-hinge loss with known margin, projections onto the data-span and the $\ell_2$ ball ensures convergence to the max-margin solution. 

\vspace{-2ex}
\section{Conclusion}
\label{sec:conclusion}
\vspace{-1ex}
For both linear regression and classification, we saw that an interpolating solution found by an optimizer minimizes an equivalent quadratic norm. This reasoning enabled us to use projections to move between norms (and therefore solutions) for over-parameterized settings. For classification, we showed that it important to consider the geometry induced by the data to measure generalization. We also proposed techniques to improve the generalization of optimization methods. We consider extending our techniques to non-linear models including deep networks as important future work. Finally, we hope to use our insights to develop optimizers that are guaranteed to generalize well. 


\begin{ack}
We would like to thank Nitin Surya for initial help with the code. This research was partially supported by the Canada CIFAR AI Chair Program, a Google Focused Research award, an IVADO postdoctoral scholarship, an NSERC CGS M award and by the NSERC Discovery Grants RGPIN-2017-06936. Simon Lacoste-Julien is a CIFAR Associate Fellow in the Learning in Machines \& Brains program
\end{ack}

\bibliographystyle{plain}
\bibliography{ref}

\newpage
\appendix
\newgeometry{margin=0.58in}
\appendixTitle

{\bf Organization of the Appendix}
\begin{itemize}
    \item[\ref{app:reg-main-proofs}] \nameref{app:reg-main-proofs}\\[.1em]

    \item[\ref{app:class-main-proofs}] \nameref{app:class-main-proofs}\\[.1em]
    
    \item[\ref{app:proofs-lin-reg}] \nameref{app:proofs-lin-reg}\\[.1em]
    
    \item[\ref{app:proofs-lin-class}] \nameref{app:proofs-lin-class}\\[.1em]
    
    \item[\ref{app:exp-regression}] \nameref{app:exp-regression}\\[.1em]
    
    \item[\ref{app:exp-classification}] \nameref{app:exp-classification}\\[.1em]
    
\end{itemize}

\section{Main proofs for linear regression}
\label{app:reg-main-proofs}
\subsection{Proof of Lemma~\ref{lemma:opt-P-reduction}}
\begin{proof}
Let $\x_{opt}$ be the solution to which a given optimizer converges to. Since $\x_{opt}$ interpolates the data, $\mx \x_{opt} = \vy$. This solution also corresponds to the min-norm solution measured in the $\mm$ norm, implying we want to find a positive definite matrix $\mm$ s.t. 
\begin{align*}
\x_{opt} = \argmin_{z} \frac{1}{2} z\transpose \,\mm \vz \quad \text{s.t.} \quad \mx z = \vy
\intertext{The Lagrangian for the optimization on the RHS can be written as follows. Here $\lambda \in \mathbb{R}^{n \times 1}$}
L(z, \lambda) = \frac{1}{2} z\transpose \,\mm \vz + \lambda\transpose \left(\mx \vz - \vy \right) \\
\frac{\partial L(\vz, \lambda)}{\partial \vz} =\mm  + \mx\transpose \lambda 
\intertext{Since $\x_{opt}$ is the solution of this optimization problem, }
\mm \x_{opt}  = - \mx\transpose \lambda
\end{align*}
implying $\mm \x_{opt} \in span\{\mx\transpose\}$. We choose $\mm$ to be the following matrix, 
\[
\mm = \normsq{\x_{opt}} \mi_d - \x_{opt} \x_{opt}\transpose + \frac{\nu \nu\transpose}{\langle \x_{opt}, \nu \rangle}
\]
Here $\nu = \mx\transpose \alpha$ where $\alpha$ is a random vector such that $\langle \x_{opt}, \nu \rangle > 0$. Note that for any $\nu$, either $\nu$ or $-\nu$ satisfies this constraint. We now verify that $\mm \x_{opt} = \nu = \mx\transpose \alpha$. 
\begin{align*}
\mm \x_{opt} = \normsq{\x_{opt}} \x_{opt} - (\x_{opt} \x_{opt}\transpose) \x_{opt} + \frac{\nu \nu\transpose \x_{opt}}{\langle \x_{opt}, \nu \rangle} = \normsq{\x_{opt}} \x_{opt} - \normsq{\x_{opt}} \x_{opt} + \nu = \nu
\end{align*}
We now compute $\va\transpose \mm \va$ to verify that $\mm$ is positive definite. 
\begin{align*}
\va\transpose \mm \va & = \normsq{\x_{opt}} \normsq{\va} - \left(\va\transpose \x_{opt}) \, (\va\transpose \x_{opt}\right)\transpose + \frac{(\va\transpose \nu) \, (\va\transpose \nu)\transpose}{\langle \x_{opt}, \nu \rangle} = \normsq{\x_{opt}} \normsq{\va} - \normsq{\va \, \x_{opt}} + \frac{\normsq{\va\transpose \nu}}{\langle \x_{opt}, \nu \rangle} \\
& > \normsq{\x_{opt}} \normsq{\va} - \normsq{\x_{opt}} \normsq{\va} + \frac{\normsq{\va\transpose \nu}}{\langle \x_{opt}, \nu \rangle} = \frac{\normsq{\va\transpose \nu}}{\langle \x_{opt}, \nu \rangle}
\end{align*}
Since $\langle \x_{opt}, \nu \rangle > 0$ by construction, $\va\transpose \mm \va > 0$ for all $\va$, implying $\mm$ is positive definite. 

The preconditioner resulting in the min-norm solution in the $\mm$ norm is $\mm^{-1}$ which is a family of preconditioners that result in the same solution as $\x_{opt}$.
\end{proof}

\subsection{Proof of Proposition~\ref{prop:proj-span}}
\begin{proof}
Recall that the solution found by an optimizer with equivalent preconditioner $\mxp_1$ can be written as \begin{align*}
\winf^{\mxp_1} & = \mxp_1 \mx\transpose (\mx \mxp_1 \mx\transpose)^{-1} \vy \\
\intertext{Projecting this solution using the projection matrix, $\pi_P = \mxp_2 \mx\transpose (\mx \mxp_2 \mx\transpose)^{-1} \mx$, }    
\pi_P [\x_{opt}] & = \mxp_2 \mx\transpose (\mx \mxp_2 \mx\transpose)^{-1} \mx \winf^{\mxp_1}  \\
& = \mxp_2 \mx\transpose (\mx \mxp_2 \mx\transpose)^{-1} \mx \mxp_1 \mx\transpose(\mx \mxp_1 \mx\transpose)^{-1}\vy \\
& = \mxp_2 \mx\transpose (\mx \mxp_2 \mx\transpose)^{-1} \vy = \winf^{\mxp_2}
\end{align*}
which is the solution obtained by PGD with a preconditioner $\mxp_2$ and interpolates the data. 
\end{proof}

\section{Main proofs for linear classification}
\label{app:class-main-proofs}

\subsection{Proof of Lemma~\ref{lemma:exp-loss-pgd}}
We map preconditioned gradient descent to the notation of~\cite{soudry2018implicit} and use Theorem 7 that considers steepest descent to prove the statement of the lemma. We first consider a constant step-size s.t. $\etak = \eta < 1/f(\x_0)$. 
\begin{proof}
PGD has the following update:
\begin{align*}
\xkk & = \xk - \eta \mxp \grad{\xk}    
\intertext{The update considered in Appendix B.2. of~\cite{soudry2018implicit} can be written as follows:}
\xkk & = \xk - \eta \gamma_k \Delta \xk \\
\intertext{where $\gamma_k$ and $\Delta \xk$ are defined as:}
\gamma_k & = \norm{\grad{\xk}}_{*} \quad \text{;} \quad \langle \Delta \xk, \grad{\xk} \rangle = \norm{\grad{\xk}}_{*} \quad \text{;} \quad \norm{\Delta \xk} = 1
\intertext{Here, $\norm{\cdot}_{*}$ is the dual norm.}
\intertext{Mapping PGD to this update, we get the following equivalence:}
\gamma_k = \indnorm{\grad{\xk}}{\mxp} \quad \text{;} \quad \Delta \xk = \frac{P\grad{\xk}}{\indnorm{\grad{\xk}}{\mxp}}
\intertext{It is easy to verify that $\norm{\cdot} = \indnorm{\cdot}{\pco}$ and $\norm{\cdot}_{*} = \indnorm{\cdot}{\mxp}$, and that $\gamma_k$ and $\Delta \xk$ satisfy the above relations.}
\end{align*}
Given this mapping, we follow the proof of Theorem 7. Using the descent lemma and the property of the exponential loss, 
\begin{align*}
f(\xkk) & \leq f(\xk) \left(1 - \frac{\eta \gamma_k}{f(\xk)} + \frac{\eta^2 \gamma_k^2}{2} \right)
\intertext{Recursing, multiplying both sides by $-\log$ and using Jensen's inequality, we get a bound on the unnormalized margin}
\min_{j} \langle \xkk, \vx_j \rangle & \geq \sum_{i = 0}^{k} \frac{\eta \gamma_i^2}{f(\x_i)} -\frac{1}{2} \eta^2 \sum_{i = 0}^{k} \gamma_i^2 - \log(f(\x_0))
\intertext{We now upper bound the $\norm{\xkk}$. Note for PGD, this norm is the induced norm wrt to the $P^{-1}$ matrix.}
\norm{\xkk} & = \norm{\x_0 - \eta \sum_{i =0}^{t} \gamma_i \Delta \x_i} \leq \norm{\x_0} + \eta \norm{\sum_{i = 0}^{t} \gamma_i \Delta \x_i} \\
& \leq \norm{\x_0} + \eta \sum_{i = 0}^{t} \norm{\gamma_i \Delta \x_i} \leq \norm{\x_0} + \eta \sum_{i = 0}^{t} \gamma_i
\intertext{Dividing the above inequalities, we get a bound on normalized margin,}
\frac{\min_{j} \langle \xkk, \vx_j \rangle}{\norm{\xkk}} & \geq \frac{\sum_{i = 0}^{k} \frac{\eta \gamma_i^2}{f(\x_i)} -\frac{1}{2} \eta^2 \sum_{i = 0}^{k} \gamma_i^2 - \log(f(\x_0))}{\norm{\x_0} + \eta \sum_{i = 0}^{t} \gamma_i}
\intertext{Using the result in Lemma 2 of~\cite{soudry2018implicit}, $\gamma_i \geq \gamma f(\x_i)$,}
\frac{\min_{j} \langle \xkk, \vx_j \rangle}{\norm{\xkk}} &  \geq \frac{\gamma \sum_{i = 0}^{k} \eta \gamma_i -\frac{1}{2} \eta^2 \sum_{i = 0}^{k} \gamma_i^2 - \log(f(\x_0))}{\norm{\x_0} + \eta \sum_{i = 0}^{t} \gamma_i}
\end{align*}
\begin{align*}
\intertext{Using the descent lemma,}
f(\xkk) & \leq f(\xk) - \frac{\eta}{2} \gamma^2_k \implies \sum_{i = 0}^{k}\frac{\eta}{2} \gamma^2_k \leq f(\x_0) & \tag{By telescoping from $i = 0$ to $k$}
\intertext{Using this bound in the above inequality,}
\frac{\min_{j} \langle \xkk, \vx_j \rangle}{\norm{\xkk}} &  \geq \frac{\gamma \sum_{i = 0}^{k} \eta \gamma_i - \eta f(\x_0) - \log(f(\x_0))}{\norm{\x_0} + \eta \sum_{i = 0}^{t} \gamma_i} \\
\intertext{Rearranging,}
& \geq \gamma - \left[\frac{\gamma \norm{\x_0} + \eta f(\x_0) + \log(f(\x_0))}{\norm{\x_0} + \eta \sum_{i = 0}^{t} \gamma_i} \right]
\intertext{Finally, we use Claim 1 in Theorem 7 of~\cite{soudry2018implicit} to bound $\eta \sum_{i = 0}^{t} \gamma_i$,}
\eta \sum_{i = 0}^{t} \gamma_i & \geq \log \left( \eta \gamma^2 (k+1) \right)
\intertext{Using this bound in the above inequality,}
\frac{\min_{j} \langle \xkk, \vx_j \rangle}{\norm{\xkk}} & \geq \gamma - \left[\frac{\gamma \norm{\x_0} + \eta f(\x_0) + \log(f(\x_0))}{\norm{\x_0} + \log \left( \eta \gamma^2 (k+1) \right)} \right]
\intertext{If $\x_0 = 0$, $f(\x_0) = 1$, $\eta = 1/f(\x_0) = 1$}
\implies \frac{\min_{j} \langle \xkk, \vx_j \rangle}{\norm{\xkk}} & \geq  \gamma - \frac{1}{\log \left(\gamma^2 (k+1) \right)} 
\intertext{Making the dependence on $P$ explicit,}
\implies \frac{\min_{j} \langle \xkk, \vx_j \rangle}{\indnorm{\xkk}{P^{-1}}} & \geq  \gamma - \frac{1}{\log \left(\gamma^2 (k+1) \right)} 
\intertext{As $k \rightarrow \infty$,}
\frac{\min_{j} \langle \xkk, \vx_j \rangle}{\indnorm{\xkk}{P^{-1}}} \rightarrow \gamma
\end{align*}
We now show that using backtracking line-search procedure to set the step-size does not change the implicit regularization and results in a similar bound. The line-search procedures picks the largest step-size that satisfies the Armijo condition:
\begin{align*}
f(\xkk) & \leq f(\xk) - c \etak \normsq{\grad{\xk}}_{*} 
\intertext{Here, $c$ is a hyper-parameter. We assume that the resulting step-size $\etak \in [\eta_{\min}, \eta_{\max}]$. Using the PGD update:}
\xkk & = \xk - \etak P \grad{\xk}  
\intertext{Following the same analysis,}
\frac{\min_{j} \langle \xkk, \vx_j \rangle}{\norm{\xkk}} &  \geq \frac{\gamma \sum_{i = 0}^{k} \eta_i \gamma_i -\frac{1}{2} \sum_{i = 0}^{k} \eta^2_i \gamma_i^2 - \log(f(\x_0))}{\norm{\x_0} + \sum_{i = 0}^{t} \eta_i \gamma_i} \\
& \geq \frac{\gamma \sum_{i = 0}^{k} \eta_i \gamma_i -\frac{\eta_{\max}}{2} \sum_{i = 0}^{k} \eta_i \gamma_i^2 - \log(f(\x_0))}{\norm{\x_0} + \sum_{i = 0}^{t} \eta_i \gamma_i}
\intertext{Using the line-search condition,}
f(\xkk) & \leq f(\xk) - c \etak \gamma_k^2 \implies \sum_{i= 0}^{k} \eta_i \gamma_i^2 \leq  \frac{f(\x_0)}{c}
\intertext{Using this bound in the above inequality,}
&  \geq \frac{\gamma \sum_{i = 0}^{k} \eta_i \gamma_i -\frac{\eta_{\max} f(\x_0)}{2c} - \log(f(\x_0))}{\norm{\x_0} + \sum_{i = 0}^{t} \eta_i \gamma_i}
\intertext{Rearranging,}
& \geq \gamma - \left[\frac{\gamma \norm{\x_0} + \frac{\eta_{\max} f(\x_0)}{2c} + \log(f(\x_0))}{\norm{\x_0} + \sum_{i = 0}^{t} \eta_i \gamma_i} \right] \geq \gamma - \left[\frac{\gamma \norm{\x_0} + \frac{\eta_{\max} f(\x_0)}{2c} + \log(f(\x_0))}{\norm{\x_0} + \eta_{\min} \sum_{i = 0}^{t} \gamma_i} \right]
\intertext{The line-search implies}
f(\xkk) & \leq f(\xk) - c \etak \gamma_k^2 \leq f(\xk) - c \etak \gamma f(\xk) \leq f(\xk) - c \eta_{\min} \gamma f(\xk)
\intertext{Using Claim 1 in Theorem 7 of~\cite{soudry2018implicit},}
c \eta_{\min} \sum_{i = 0}^{k} \gamma_i & \geq \log \left(c \eta_{\min} \, \gamma^2 (k+1) \right)
\intertext{Using this bound with the above inequality,}
\frac{\min_{j} \langle \xkk, \vx_j \rangle}{\norm{\xkk}} & \geq \gamma - \left[\frac{\gamma \norm{\x_0} + \frac{\eta_{\max} f(\x_0)}{2c} + \log(f(\x_0))}{\norm{\x_0} + \frac{\log \left(c \eta_{\min} \, \gamma^2 (k+1) \right)}{c}} \right]
\intertext{For line-search, $\eta_{\min} = \frac{2 (1-c)}{L}$,}
& \geq \gamma - \left[\frac{\gamma \norm{\x_0} + \frac{\eta_{\max} f(\x_0)}{2c} + \log(f(\x_0))}{\norm{\x_0} + \frac{\log \left(2 c (1-c) \, \gamma^2 (k+1) / L\right)}{c}} \right] \\
\intertext{If $\x_0 = 0$, $f(\x_0) = 1$,}
\frac{\min_{j} \langle \xkk, \vx_j \rangle}{\indnorm{\xkk}{P^{-1}}} & \geq \gamma - \left[\frac{\eta_{\max}}{2 \log \left(2 c (1-c) \, \gamma^2 (k+1) / L \right)} \right]
\intertext{And as $k \rightarrow \infty$,}
\frac{\min_{j} \langle \xkk, \vx_j \rangle}{\indnorm{\xkk}{P^{-1}}} \rightarrow \gamma
\end{align*}
\end{proof}

\subsection{Proof of Lemma~\ref{lemma:exp-equivalent-precond}}
\begin{proof}
Given $\x_{opt}$, the iterate obtained after running an arbitrary optimizer for $T$ iterations, we want to find matrix $\mm$ such that
\begin{align*}
\frac{\x_{opt}}{\norm{\x_{opt}}} & = \argmax_{\indnorm{\vz}{\mm} \leq 1} \min_{i} \langle \vz, \vx_i \rangle
\intertext{We constrain $\mm$ such that $\indnorm{\x_{opt}}{\mm} = 1$. If $\mm$ satisfies this constraint, then}
\x_{opt} & = \argmax_{\indnorm{\vz}{\mm} \leq 1} \min_{i} \langle \vz, \vx_i \rangle
\intertext{For the given $\x_{opt}$, there exist a unique set of support vectors $S$ s.t for $\vs \in S$, $\min_i \langle \x_{opt}, \vs \rangle = \min_i \langle \x_{opt}, \vx_i \rangle$. Simplifying, matrix $\mm$ should satisfy the following equality for $\vs \in S$,}
\x_{opt} & = \argmax_{\indnorm{\vz}{\mm} \leq 1} \langle \vz, \vs \rangle \implies \x_{opt}  = \argmax_{\indnorm{\vz}{\mm} \leq 1} \sum_{\vs \in S} \langle \vz, \vs \rangle \\
\intertext{The Lagrangian for the RHS can be written as:}
L(\vz, \lambda)  & = \sum_{\vs \in S} \langle \vz, \vs \rangle + \lambda \left(\indnorm{\vz}{\mm}^2 - 1 \right) \\
\frac{\partial L(\vz, \lambda)}{\partial \vz} = 0 \implies \mm \x_{opt} & = - \frac{\sum_{\vs \in S} \vs}{\lambda} \intertext{Implying that vector $\mm \x_{opt}$ lies in the span of the support vectors. Let $\mv \in \mathbb{R}^{\vert S \vert \times d}$ be the matrix of support vectors. And let $\alpha$ be an $\vert S \vert$-dimensional vector of coefficients.}
\implies \mm \x_{opt} & = \mv\transpose \alpha 
\end{align*}
We now use the norm constraint $\indnorm{\x_{opt}}{\mm} = 1$ to constrain the $\alpha$. Specifically, we want, $\indnormsq{\x_{opt}}{\mm} = \x_{opt}\transpose V\transpose \alpha = 1$. This implies that $\alpha$ should satisfy the following equality that ensures $\indnorm{\x_{opt}}{\mm} = 1$,
$$
\langle \mv \x_{opt}, \alpha \rangle = 1
$$
Since $\mm$ is the inverse of a preconditioner matrix, it needs to be positive definite. Using the same construction as in Lemma~\ref{lemma:opt-P-reduction} with the additional constraint, 
\begin{align}
\mm = \normsq{\x_{opt}} \mi_d - \x_{opt} \x_{opt}\transpose + \frac{(\mv\transpose \alpha) (\mv\transpose \alpha)\transpose}{\langle \x_{opt}, (\mv\transpose \alpha) \rangle} \\
\text{where} \quad \mv\transpose \alpha > 0 \quad \text{;} \quad \langle \mv \x_{opt}, \alpha \rangle  = 1.   
\end{align}
Similar to Lemma~\ref{lemma:opt-P-reduction}, we can verify that $\mm$ is positive definite if $\mv\transpose \alpha > 0$ and satisfies the equality $\mm \x_{opt} = \mv\transpose \alpha$. 
\end{proof}

\subsection{Proof of Lemma~\ref{lemma:RC_Precond}}
\label{sec:RC_Precond}
\begin{proof}
Based on the definition of Rademacher complexity we have 
\begin{align}
    \hat{\mathcal R}(\mathcal F)& = \mathbb E_\sigma \Big[ \sup_{f\in \mathcal F} |\frac{2}{n} \sum_{i=1}^n \sigma_i y_i\vw\transpose \vx_i|\Big] = \mathbb E_\sigma \Big[ \sup_{\vw:1/2\vw\transpose \inv{\mxp} \vw \leq E } |\frac{2}{n} \sum_{i=1}^n \sigma_i y_i\vw\transpose \vx_i|\Big]\\
    & \leq \frac{2\sqrt{2 E}}{n} \mathbb E_\sigma \Big[\big(\sum_{i=1}^n \sigma_i y_i \vx_i\transpose \mxp  \sum_{j=1}^n \sigma_j y_j \vx_j\big)^{1/2}  \Big]\\
    &\leq \frac{2\sqrt{2 E}}{n}  \Big[\big(\mathbb E_\sigma \sum_{j,i=1}^n \sigma_i  \sigma_j y_i   y_j\vx_i\transpose \mxp  \vx_j\big)^{1/2}  \Big] = \frac{2\sqrt{2 E}}{n} \sqrt{tr(\mxp \mk)}
\end{align}
To find the optimal $\mxp$, we add the constrain $ln(\det(\mxp)) > 0 $ which guarantees that $\mxp$ doesn't have any zero eigenvalue. So the objective function is 
\begin{equation}
    \argmin_{\mxp} tr(\mxp \mk ) - ln (\det (\mxp))
\end{equation}
Taking derivative w.r.t. $\mxp$ and set it to zero we have
\begin{equation}
    \mxp^* = \inv{\mk}.
\end{equation}
We can assume that $\mk$ is symmetric positive definite. Furthermore, if we assume $\mk_+$ and $\mk_-$ are scaled covariance matrix for data belonging to $+1$ and $-1$ classes, we get the following upper bound for the Rademacher complexity
\begin{equation}
    \hat{\mathcal R}(\mathcal F) \leq \frac{2\sqrt{2 E}}{n} \sqrt{tr(\mxp (\mk_+ + \mk_- ))}. 
\end{equation}
Using the above argument, then the optimal precondition is 
\begin{equation}
    \mxp^* = \inv{(\mk_+ + \mk_-) }.
\end{equation}
Similarly we can find an upper bound for the VC-dimension of max-margin which has the same upper bound as we get for the Rademacher complexity. For this part we assume that there exist a data set $\mathcal D_v$ such that max-margin model can shatter it  and besides for any positive definite $\mxp$ we have $tr(\mxp \mk_v) \leq tr(\mxp \mk)$ where $\mk_v$ is $\mx_v\transpose \mx_v$. To be more precise, let assume $v$ is the VC-dimension of max-margin problem, therefore there exist an a set of size $v$ i.e. $\mathcal D_v =\{(\vx_j,y_j)\}_{j=1:v}$ such that for all $j$ $y_j \vw\transpose \vx_j \geq 1$. If we sum up the both side of this inequality for all $j$ data point we have: 
\begin{align*}
        v &\leq \sum_{j=1}^v y_j \vw\transpose \vx_j \\
        & \leq \|\vw\|_{\inv{\mxp}} \Big[\big(\sum_{j=1}^v y_j \vx_i\transpose \mxp  \sum_{k=1}^v y_k \vx_k\big)^{1/2}  \Big]\\
        & \leq \sqrt{2E}\Big[\big(\sum_{j=1}^v y_j \vx_i\transpose \mxp  \sum_{k=1}^v y_k \vx_k\big)^{1/2}  \Big]
\end{align*}
where the second inequality is due to Cauchy-Schwarz.  Since max-margin can shatter this set, we can assume that $y_j's$ are independent random variables, and take expectation from both sides of the above bound and then apply the Jensen inequality to get the following 
\begin{align}
    v &\leq \sqrt{2E}\Big[\mathbb E \big(\sum_{j=1}^v y_j \vx_i\transpose \mxp  \sum_{k=1}^v y_k \vx_k\big) \Big]^{1/2}\\
   & \leq \sqrt{2E} \Big[ \sum_{j=1}^v \vx_i\transpose \mxp \vx_i \Big]^{1/2} \\
    & \leq \sqrt{2E} \sqrt{tr(\mxp \mk_v )} \leq \sqrt{2E} \sqrt{tr(\mxp \mk )}
\end{align}
    where the second inequality due to the fact that $\mathbb E(y_j) = 0$ and $y_j$ and $y_k$ are independent. The last inequality is due to our assumption.
    Here we show that when $\mathcal D_v \subset \mathcal D$ which means there is a subset of training data with size $v$ such that max-margin can shatter it, then the assumption $tr(\mxp \mk_v) \leq tr(\mxp \mk)$. Let $\ms \in \mathbb R^{n\times n}$ be a diagonal matrix with $\ms_{i,i} =1$ if $\vx_i \in \mathcal D_v$ and $0$ everywhere else. Therefore $\ms \mx$ represents data points in $\mathcal D_v$ and we have $\mk_v = \mx\transpose \ms \mx$. Let $\mm = \mx \mxp \mx\transpose$ and we have $tr (\mm) = tr(\mxp \mk)$. We observe that $tr(\mxp \mk_v)= tr(\mm \ms)$. Since $\mxp$ is positive definite, all the diagonal elements of $\mm$ are positive. Therefore multiplying $\mm$ by $\ms$ sets some of the diagonal elements to $0$, therefore we have $tr(\mxp \mk_v) =tr(\mm \ms) \leq tr(\mm) = tr(\mxp \mk)$. 
\end{proof}

\subsection{Proof of Lemma~\ref{lemma:over-lin-class-equivalence}}
\label{app:over-lin-class-equivalence}
\begin{proof}
Let $\mx \prp \in \mathbb{R}^{(d-n) \times (d)}$ represents the components orthogonal to the data span, then any preconditioner can be decomposed as: $\mxp = \mx \transpose \mxp_1 \mx + \mx\transpose \mxp_2 \mx \prp + \mx \transpose\prp \mxp_2\transpose \mx + \mx \transpose\prp \mxp_3 \mx \prp$. The first component lies in the data span, whereas the last component lies outside the data span. The component $\mx\transpose \mxp_2 \mx \prp$ is an operator that takes a component perpendicular to the span and brings it into the data span, whereas the component $\mx \transpose\prp \mxp_2\transpose \mx$ does the reverse. 
Since the gradient lies in the span, for any vector $\vw$, we have 
\[\mxp \grad{\vw} = (\mx \transpose \mxp_1 \mx + \mx \transpose\prp \mxp_2\transpose \mx) \grad{\vw}\]
 since $\mx \prp$ is orthogonal to the data span. Suppose we run PGD from $\vw_0 = 0$, implying that $\vw_1 = - \eta \mxp \grad{\vw_0} = -\eta (\mx \transpose \mxp_1 \mx + \mx \transpose\prp \mxp_2\transpose \mx) \grad{\vw_0} = (\vw_1)\prp + (\vw_1)\parll$, where $(\vw_1)\prp = -\eta (\mx \transpose\prp \mxp_2\transpose \mx) \grad{\vw_0}$ and $(\vw_1)\parll = -\eta (\mx \transpose \mxp_1 \mx) \grad{\vw_0}$. By projecting after every gradient step, we compute $\bar{\vw}_1 = \pi [\vw_1] = -\eta (\mx \transpose \mxp_1 \mx) \grad{\vw_0}$. It is easy to see that one iteration of PGD with a preconditioner $\bar{\mxp} = \mx \transpose \mxp_1 \mx$ and $\vw_0 = 0$ would result in the iterate $\bar{\vw}_1$ where $\bar{\vw}_1$ lies in the data span and $\mxp_1 = (\mx \mx \transpose)^{-1} \mx \mxp \mx \transpose (\mx \mx \transpose)^{-1}$. The same reasoning can be used recursively, implying that projected PGD with a preconditioner $\mxp$ in the over-parameterized setting is equivalent to PGD with a preconditioner $\mx \transpose \mxp_1 \mx$ in the under-parameterized. We can prove it by induction. Since $\vw_0= 0$, we have $\pi[\vw_0]  = \bar{\vw}_0 =0 $. Now assume for iteration $t$ we have $\pi(\vw_t)  = \bar{\vw}_t$. Now based on the iteration of projected PGD we have 
\[
 \vw_{t+1} = \pi[\vw_t] - \eta_t \mxp \nabla f(\pi[\vw_t]) = \bar{\vw}_t - \eta_t \mxp \nabla f(\bar{\vw}_t).  
 \]
 Also for $\bar{\vw}_{t+1}$ which is the next iterate of PGD with $\bar{\mxp}$ as preconditioner we have 
\[
 \bar{\vw}_{t+1} = \bar{\vw}_t - \eta_t \bar{\mxp} \nabla f(\bar{\vw}_t).  
 \] 
 Now if we project $\vw_{t+1}$ into the data span we have 
 \[
	\pi [\vw_{t+1}] = \pi [\bar{\vw}_t - \eta_t \mxp \nabla f(\bar{\vw}_t)] = \bar{\vw}_t -\eta_t  \pi [\mxp \nabla f(\bar{\vw}_t)] = \bar{\vw}_t -\eta_t   \bar{\mxp} \nabla f(\bar{\vw}_t)=  \bar{\vw}_{t+1} 
 \]
 which complete the proof. 
\end{proof}

\subsection{Proof of Proposition~\ref{thm:sqhinge-projgd-rate}}
\label{app:}
\begin{proof}
Since, the function is smooth and convex, projected GD can obtain the following rate using gradient descent:
\begin{align*}
f(\x_T) - f^* \leq \frac{L \normsq{\x_0 - \wstar}}{T}
\end{align*}
where $f(\x) = \frac{1}{n}\sum_{i = 1}^{n} \left(\max\{1 - y_i \langle \x, \vx_i \rangle, 0\} \right)^{2}$. W.l.o.g. assume that $\gamma =1 $. Since the data is linearly separable by a margin $1$, and we project on the $\ell_2$ ball with radius $1$, $f^* = 0$ since $y_i \langle \xopt, \vx_i \rangle \geq 1$. The projection also ensures that $\norm{\x_t} = 1$ for all iterates $\x_t$.

Denoting $L \normsq{\x_0 - x^*}$ as $c$, we obtain the following bound,
\begin{align*}
\frac{1}{n}\sum_{i = 1}^{n} \left(\max\{1 - y_i \langle \x_T, \vx_i \rangle, 0\} \right)^{2} & \leq \frac{c}{T} \\
\intertext{Taking square-root on both sides and using Jensen's inequality,}
\frac{1}{n}\sum_{i = 1}^{n} \left(\max\{1 - y_i \langle \x_T, \vx_i \rangle, 0\} \right) & \leq \sqrt{\frac{c}{T}} \\
\implies \sum_{i = 1}^{n} \left(\max\{1 - y_i \langle \x_T, \vx_i \rangle, 0\} \right) & \leq n\, \sqrt{\frac{c}{T}} \\
\intertext{Let us lower bound the LHS,}
\sum_{i = 1}^{n} \left(\max\{1 - y_i \langle \x_T, \vx_i \rangle, 0\} \right) \geq \max_{i \in [n]} \{ \left(\max\{1 - y_i \langle \x_T, \vx_i \rangle, 0\} \right) \}
\intertext{Let $S = \{j \vert y_i \langle \x_T, \vx_i \rangle \geq 1\}$ be the set of points classified with margin $1$. Denoting $\bar{S}$ as the complement of this set.}
\max_{i \in [n]} \{ \left(\max\{1 - y_i \langle \x_T, \vx_i \rangle, 0\} \right) \} & \geq \max_{i \in \bar{S}} \{ \left(1 - y_i \langle \x_T, \vx_i \rangle \right) \} = 1 - \min_{i \in \bar{S}} y_i \langle \x_T, \vx_i \rangle \\
\implies 1 - \min_{i \in \bar{S}} y_i \langle \x_T, \vx_i \rangle & \leq n \sqrt{\frac{c}{T}} \\
\implies \min_{i \in \bar{S}} y_i \langle \x_T, \vx_i \rangle \geq 1 - n \sqrt{\frac{c}{T}} 
\intertext{Since $\min_{i \in \bar{S}} y_i \langle \x_T, \vx_i \rangle \geq \min_{i \in S} y_i \langle \x_T, \vx_i \rangle$,}
\min_{i \in [n]} y_i \langle \x_T, \vx_i \rangle \geq 1 - n \sqrt{\frac{c}{T}} \\
\implies \frac{\min_{i \in [n]} y_i \langle \x_T, \vx_i \rangle}{\norm{\x_T}} \geq \gamma - \gamma n \sqrt{\frac{c}{T}} &\tag{Since $\norm{\x_T} = 1$ and the true margin $\gamma = 1$,} \\
\hat{\gamma} \geq \gamma - \gamma n \sqrt{\frac{c}{T}}
\end{align*}
This rate can be improved by using Nesterov acceleration. Following the same proof, we obtain the following bound:
\begin{align*}
\hat{\gamma} & \geq \gamma - \gamma \frac{n \sqrt{c}}{T} \end{align*}
\end{proof}

\section{Additional proofs for linear regression}
\label{app:proofs-lin-reg}

\subsection{Convergence of preconditioned gradient descent}
\label{app:pgd}
\begin{lemma}
When optimizing the squared loss, the iterates of PGD with preconditioner $\mxp$ and constant step-size $\eta \leq \frac{1}{\lambda_{max} (\mx \mxp \mx\transpose)}$ evolve as:
\begin{align*}
\vw_{PGD} & = \lim_{k \to \infty} \vw_k = \vw_0 + \mxp \mx\transpose (\mx \, \mxp \, \mx\transpose)^{-1} [\vy - \mx \vw_0] 
\intertext{Furthermore, $w_{PGD}$ is the solution to a constrained minimization problem,}
\vw_{PGD} & = \argmin \frac{1}{2} \|\vw-\vw_0\|_{\pco}^2 \quad \text{s.t.} \quad \ \mx \vw = \vy
\end{align*}
\label{lemma:pgd-iterates}
\end{lemma}
\begin{proof}
The PGD update for linear-regression can be written as:
\begin{align*}
\xkk & = \xk - \eta \mxp \grad{\xk} = \xk - \eta \mxp \mx\transpose (\mx \xk - \vy) 
\intertext{Starting at $\x_0$ and defining $\vy_0 = \mx \vw_0$.}
\x_1 & = \x_0 - \eta \mxp \mx\transpose (\mx \x_0 - \vy) = \x_0 - \eta \mxp\mx\transpose [\vy_0 - \vy]
\intertext{Further unfolding the iterates,}
\x_2 & = \x_1 - \eta \mxp\mx\transpose (\mx \x_1 - \vy) 
= \x_0 - \eta \mxp \mx\transpose [\vy_0 - \vy] + \eta \mxp \mx\transpose \vy - \eta \mxp \mx\transpose (\mx \x_1) \\
\intertext{Adding and subtracting $\eta \mxp \mx\transpose [\vy_0 - \vy]$} 
& = \x_0 - 2 \eta \mxp \mx\transpose [\vy_0 - \vy] + \eta \mxp \mx\transpose [\vy_0 - \vy] + \eta \mxp \mx\transpose \vy - \eta \mxp \mx\transpose (\mx \x_1) \\
& = \x_0 - 2 \eta \mxp \mx\transpose [\vy_0 - \vy] + \eta \mxp \mx\transpose \vy_0 - \eta \mxp \mx\transpose \left(\mx ( \x_0 - \eta \mxp \mx\transpose [\vy_0 - \vy])\right) \\
& = \x_0 - 2 \eta \mxp \mx\transpose [\vy_0 - \vy] + \eta \mxp \mx\transpose \vy_0 - \eta \mxp \mx\transpose \vy_0 + \eta^2 \mxp \mx\transpose \mx \mxp \mx\transpose [\vy_0 - \vy] \\
& = \x_0 - 2 \eta \mxp \mx\transpose [\vy_0 - \vy] +  \eta^2\mxp \mx\transpose \mx \mxp \mx\transpose [\vy_0 - \vy] \\
& = \x_0 - \mxp \mx\transpose \left[2 \eta - \eta^2 (\mx \mxp \mx\transpose) \right] \, [\vy_0 - \vy] \\
\intertext{Defining $\mk = \mx \mxp \mx\transpose$}
\x_2 & = \x_0 - \mxp \mx\transpose \left[2 \eta - \eta^2 \mk \right] \, [\vy_0 - \vy] \\
\intertext{Similarly writing down $\vw_3$,}
\x_3 & = \x_2 - \eta \mxp \mx\transpose (\mx \x_2 - \vy) = \x_2 + \eta \mxp \mx\transpose y -  \eta \mxp \mx\transpose \mx \x_2 \\
& = \x_0 - \mxp \mx\transpose \left[2 \eta - \eta^2  \mk \right] \, [\vy_0 - \vy] + \eta \mxp \mx\transpose \vy -  \eta \mxp \mx\transpose \mx \x_2 \\
\intertext{By adding, subtracting $\eta \mxp \mx\transpose \vy_0$,}
& = \x_0 - \mxp \mx\transpose \left[3 \eta - \eta^2  \mk \right] \, [\vy_0 - \vy] - \eta \mxp \mx\transpose \vy_0 -  \eta \mxp \mx\transpose \mx \x_2 \\
& = \x_0 - \mxp \mx\transpose \left[3 \eta - \eta^2  \mk \right] \, [\vy_0 - \vy] + \eta \mxp \mx\transpose \mx \mxp \mx\transpose \left[2 \eta - \eta^2  \mk \right] \, [\vy_0 - \vy] \\
& = \x_0 - \mxp \mx\transpose \left[3 \eta - \eta^2  \mk \right] \, [\vy_0 - \vy] + \eta \mxp \mx\transpose  \mk \left[2 \eta - \eta^2  \mk \right] \, [\vy_0 - \vy] \\
& = \x_0 - \mxp \mx\transpose \left[ 
3 \eta - \eta^2  \mk + 2 \eta^2  \mk - \eta^3  \mk^2 
\right] \, [\vy_0 - \vy] \\
\x_3 & = \x_0 - \mxp \mx\transpose \left[3 \eta + \eta^2  \mk - \eta^3  \mk^2 \right] \, [\vy_0 - \vy] \\
\end{align*}
Writing down the general form, 
\begin{align*}
\x_1 & = \x_0 - \mxp \mx\transpose [K^0] \, [\vy_0 - \vy] \\ 
\x_2 & = \x_0 - \mxp \mx\transpose [2 \eta  \mk^0 - \eta^2  \mk^1] \, [\vy_0 - \vy] \\
\x_3 & = \x_0 - \mxp \mx\transpose \left[3 \eta  \mk^0 + \eta^2  \mk^1 - \eta^3  \mk^2 \right] \, [\vy_0 - \vy] \\
\implies \x_k & = \x_0 - \mxp \mx\transpose \left[ 
\sum_{i = 1}^{k} (-1)^{i - 1} \binom{k}{i} \eta^{i}  \mk^{i - 1} \right] \, [\vy_0 - \vy]
\end{align*}
Using the fact that $\left[ 
\sum_{i = 1}^{k} (-1)^{i - 1} \binom{k}{i} \eta^{i}  \mk^{i - 1} \right] = -K^{-1} \left[\left(   \mi_n - \eta  \mk \right)^{k} -   \mi_n \right]$,
\begin{align*}
\x_k & = \x_0 + \mxp \mx\transpose  \mk^{-1} \left[\left(   \mi_n - \eta  \mk \right)^{k} -   \mi_n \right] \, [\vy_0 - \vy] 
\intertext{If $\eta < \frac{1}{\lambda_{max}(K)}$ and as $k \rightarrow \infty$,}
\lim_{k \rightarrow \infty} \x_k & = \x_0 - \mxp \mx\transpose  \mk^{-1} \, [\vy_0 - \vy] \\
\implies \winf & = \x_0 + \mxp \mx\transpose (\mx \mxp \mx\transpose)^{-1}\, [\vy - \mx \x_0]
\end{align*}
\end{proof}

\subsection{Properties of PGD solution}
\label{apdx:prop-of-pgd}
Similar to GD, zero initialized PGD converges to a unique solution lying in the span of the transformed data. Specifically, we obtain the following lemma:
\begin{lemma}
The solution found by PGD initialized the origin is the unique point satisfying the constraints: (i) lies in the span of the feature vectors, $\winf = \argmin_{\vz \in \text{span}({P X\transpose})} \vert \vert \winf - \vz \vert \vert^2_{\pco}$ and (ii) interpolates the data, implying that $\mx \winf = \vy$. 
\label{lemma:unique-span-gen}
\end{lemma}
\begin{proof}
\begin{align*}
\winf & = \argmin_{\vz = \mxp \mx\transpose \alpha} \vert \vert \winf - \vz \vert \vert^2_{\pco} \\
\intertext{Let $\winf = \mxp \mx\transpose \alpha_{\infty}$.}
\implies \alpha_{\infty} & = \argmin_{\alpha} \vert \vert \winf - \mxp \mx\transpose \alpha \vert \vert^2_{\pco} = \argmin_{\alpha} \left[
\winf\transpose \pco \winf -  2 \winf\transpose \mx\transpose \alpha + \alpha\transpose \mx \mxp\transpose \mx\transpose \alpha \right] \\
\implies \mx \winf & =  \alpha_{\infty}
\intertext{Since $\vw\inf$ interpolates the data, $\mx \winf = \vy$,}
\implies \alpha_{\infty} & = (\mx \mxp \mx\transpose)^{-1} \vy 
\intertext{Since, $\winf = \mxp \mx\transpose \alpha_{\infty},$}
\implies \winf & = \mxp \mx\transpose (\mx \mxp \mx\transpose)^{-1} \vy
\end{align*}
\end{proof}
The above lemma implies that the PGD solution interpolates the data and is a projection onto $\mxp \mx\transpose$ but with the distance measured in the $\pco$ norm. When $\mxp = \mi_d$, the solution is a projection onto the span of the data points measured in the $l_2$ norm, recovering the known result for GD. The above lemma show that the solutions of both GD and PGD are unique interpolating solutions in their respective subspaces. 

\begin{lemma}
The solution to the min $\mxp^{-1}$ norm least square i.e. $w_{opt} = \argmin_{w} \|w-w_0\|_{\mxp^{-1}}$ is equal to the solution of PGD when the model interpolate the data i.e. $\mx\vw = \vy$. 
\end{lemma}
\begin{proof}
To solve this we reformulate our objective in the Lagrangian form where $\lambda \in R^{n \times 1}$:
\begin{align*}
    & \mathcal L = 1/2 \|\vw-\vw_0\|_{\pco}^2 - \mathbf{\lambda}\transpose(\vy - \mx \vw)\\
    & \frac{\partial \mathcal L }{\partial \vw} = 0 \implies \pco(\vw-\vw_0)+\mx\transpose\mathbf{\lambda} = 0 \implies \mathbf{\lambda}\transpose = -(\vw-\vw_0)\transpose \mx\transpose(\mx \mxp \mx\transpose)^{-1}\\
    &\text{Let } \ma = \mx\transpose(\mx \mxp \mx\transpose)^{-1}\\
    & \implies \mathcal L = 1/2 \|\vw-\vw_0\|_{\pco}^2 + (\vw-\vw_0)\transpose \ma(\vy-\mx \vw) \\
    & \frac{\partial \mathcal L }{\partial \vw} = 0 \implies \pco(\vw-\vw_0) + \ma \vy - (2\ma \mx)\vw + \mx\transpose \ma\transpose \vw_0 =0 \\
    \intertext{Let $\winf$ be the solution of the above inequality.}
    &\implies (\winf-\vw_0) + \mxp \ma \vy - (2\mxp \ma \mx )\winf+ \mxp \ma \mx \vw_0 =0\\
    & \text{$\mx \winf=\vy$ because of the constraint. Let $\vy_0= \mx \vw_0$} \implies \winf-\vw_0 + \mxp \ma\vy -2\mxp \ma\vy + \mxp \ma\vy_0 = 0\\ 
    & \implies \winf = \vw_0 + \mxp \ma(\vy-\vy_0) = \vw_0 + \mxp \mx\transpose(\mx \mxp \mx\transpose)^{-1} [\vy - \mx \vw_0] \\
    & \implies \winf = [\mi_d - \mxp \mx\transpose(\mx \mxp \mx\transpose)^{-1} \mx] \vw_0 + \mxp \mx\transpose(\mx \mxp \mx\transpose)^{-1} \vy
\end{align*}
\end{proof}

\subsection{Convergence of Newton method}
\label{app:newton}
For the Newton method, for which the iterates can be written as: $\vw_{k+1} = \vw_k - \eta [\nabla^{2} f(\vw_k) + \lambda \mi_d]^{-1} \nabla f(\vw_k)$. Here, $\nabla^{2} f(\vw_k)$ is the Hessian equal to $\mx\transpose \mx$ for linear regression and $\lambda$ is the LM regularization parameter. 
\begin{lemma}
The Newton method remains in the span of the data-points and hence converges to the min-norm solution.
\label{lemma:newton}
\end{lemma}
\begin{proof}
Recall that $\mx \in \mathbb{R}^{n\times d}$. The singular value decomposition of $\mx = \mmu \ms \mv\transpose$, where $\mmu \in \mathbb{R}^{n\times d}$, $\ms \in \mathbb{R}^{d\times d}$ with rank $n$ and $\mv \in \mathbb{R}^{d \times d}$. 

The LM-regularized Hessian can be written as, 
\[
\mh = \mx\transpose \mx + \lambda \mi_d = \mv \ms \mv\transpose + \lambda \mv \mv\transpose = \mv (\ms + \lambda \mi_d) \mv\transpose
\]
The gradient at $\xk$ for linear regression can be written as,
\[
\grad{\xk} = \mx\transpose (\mx \xk - \vy) = \mx\transpose \gamma_k
\]
where $\gamma_k = (\mx \xk - \vy)$, implying that the gradient lies in the span of the data points. 

Let us consider the first iteration of the Newton method starting from $\x_0 = 0$,
\begin{align*}
\x_1 & = \mh^{-1} \grad{\x_0}  = \mh^{-1} \mx\transpose \gamma_{0} = \mh^{-1} \grad{\x_0} = \mh^{-1} \mx\transpose \gamma_{0} \\
& = \left[\mv (\ms + \lambda \mi_d) \mv\transpose\right]^{-1} \mx\transpose \gamma_{0} = \mv (\ms + \lambda \mi_d)^{-1} \mv\transpose  \mx\transpose \gamma_{0} \\
& = \mv (\ms + \lambda \mi_d)^{-1} \mv\transpose \mv \ms \mmu\transpose \gamma_{0} \\
\x_1 & = \mv (\ms + \lambda \mi_d)^{-1} \ms \mmu\transpose \gamma_{0} 
\end{align*}
We now show that $\x_1$ lies in the span of the data points, i.e. for $\x_1$ can be expressed as $\mx\transpose. \beta$ for $\beta \in \mathbb{R}^{n}$. We compute the value of $\beta$ below:
\begin{align*}
\mv (\ms + \lambda \mi_d)^{-1} \ms \mmu\transpose \gamma_{0} & = \mx \transpose \beta = \mv \ms \mmu\transpose \beta
\implies \ms \mmu\transpose \beta \\ 
& = (\ms + \lambda \mi_d)^{-1} \ms \mmu\transpose \gamma_{0} & \tag{Since $\mv$ is full rank, multiplying both sides by $\mv^{-1}$} \\
\intertext{Multiplying both sides by the pseudo-inverse of the rank $n$ matrix $\ms$.}
\ms^{\dagger} \ms \mmu\transpose \beta & = \ms^{\dagger} (\ms + \lambda \mi_d)^{-1} \ms \mmu\transpose \gamma_{0} \\
\intertext{Define $\mw = \ms^{\dagger} \ms \mmu\transpose$, $W \in \mathbb{R}^{d \times n}$ and multiplying by $\mw\transpose$,}
\mw\transpose \mw \beta & = \mw\transpose  \, \ms^{\dagger} (\ms + \lambda \mi_d)^{-1} \ms \mmu\transpose \gamma_{0} \\
\intertext{Multiplying by $(\mw\transpose \mw)^{-1}$ since $\mw \transpose \mw \in \mathbb{R}^{n\times n}$ is a full-rank matrix.}
\beta & = (\mw\transpose \mw)^{-1} \mw\transpose  \, \ms^{\dagger} (\ms + \lambda \mi_d)^{-1} \ms \mmu\transpose \gamma_{0}
\end{align*}
This implies that $\x_1 = \mx\transpose \beta$ for $\beta = (\mw\transpose \mw)^{-1} \mw\transpose  \, \ms^{\dagger} (\ms + \lambda \mi_d)^{-1} \ms \mmu\transpose \gamma_{0}$ and hence, $\x_1$ lies in the span of the data points. 

For the next iterate, $\x_2 = \x_1 -  \mh^{-1} \grad{\x_1}$, $\mh^{-1} \grad{\x_1}$ lies in the span by the same argument as above, and since $\x_1$ lies in the span, $\x_2$ also lies in the span.

Using the same argument for the subsequent iterates, we conclude that $\winf$ that interpolates the training data also lies in the span. 
\end{proof}

\subsection{Convergence of full-matrix Adagrad}
\label{app:adagrad}
For full matrix Adagrad~\cite{agarwal2019efficient}, the iterates can be written as: $\vw_{k+1} = \vw_k - \etak \mg_k \nabla f(\vw_k)$ where $\mg_k = \ms_k^{-1/2}$ and $\ms_k = \ms_{k-1} + (\nabla f(\vw_{k-1}) \nabla f(\vw_{k-1})^{T})$. The more commonly used diagonal version of Adagrad uses $\mg_k = diag(\ms_k)$. We obtain the following lemma, analyzing the convergence of these variants. 
\begin{lemma}
The iterates of full matrix Adagrad lie in the span of the data and converges to the min-norm solution. However, convergence to the min-norm solution is not ensured for the diagonal version of Adagrad.  
\label{lemma:adagrad}
\end{lemma}
\begin{proof}
Recall that the full-matrix Adagrad update can is given as:
\begin{align*}
\xkk = \xk - \etak \left(\sum_{j = 0}^{k} \grad{\x_j} \grad{\x_j}\transpose + \epsilon \mi_d \right)^{-1/2} \grad{\xk}
\end{align*}
Let us consider the first iteration starting from $\x_0 = 0$. In this case, 
\begin{align*}
\xkk = - \etak \left(\grad{\x_0} \grad{\x_0}\transpose + \epsilon \mi_d \right)^{-1/2} \grad{\x_0}
\end{align*}
The gradient $\grad{\x_0} = \mx\transpose \gamma_0$ as in the previous lemma. Using the singular value decomposition of $\mx = \mmu \ms \mv\transpose$
\begin{align*}
\left(\grad{\x_0} \grad{\x_0}\transpose + \epsilon \mi_d \right)^{-1/2} & = \left( \mv \ms \mmu\transpose \gamma_0 \gamma_0\transpose \mmu \ms \mv\transpose + \epsilon \mi_d \right)^{-1/2} \\
\intertext{$\gamma_0 \gamma_0\transpose$ is a rank $1$ matrix. Using its eigen-decomposition, $\gamma_0 \gamma_0\transpose = \mw \Lambda \mw\transpose$, where $\mw, \Lambda \in \mathbb{R}^{n \times n}$,}
& = \left( \mv \ms \mmu\transpose \mw \Lambda \mw\transpose  \mmu \ms \mv\transpose + \epsilon \mi_d \right)^{-1/2} \\
\intertext{Define $ma = \ms \mmu\transpose \mw \in \mathbb{R}^{d. \times n}$}
& = \left( \mv \ma \Lambda \ma\transpose \mv\transpose + \epsilon \mi_d \right)^{-1/2} = \left( \mv \left(\ma \Lambda \ma\transpose + \epsilon \mi_d \right) \mv\transpose \right)^{-1/2} \\
\implies \left(\grad{\x_0} \grad{\x_0}\transpose + \epsilon \mi_d \right)^{-1/2} & = \mv \left(\ma \Lambda \ma\transpose + \epsilon \mi_d \right)^{-1/2} \mv\transpose 
\end{align*}
\begin{align*}
\left(\grad{\x_0} \grad{\x_0}\transpose + \epsilon \mi_d \right)^{-1/2} \grad{\x_0} & = \mv \left(\ma \Lambda \ma\transpose + \epsilon \mi_d \right)^{-1/2} \mv\transpose \mx\transpose \gamma_0 \\
& = \mv \left(\ma \Lambda \ma\transpose + \epsilon \mi_d \right)^{-1/2} \mv\transpose \mv \ms \mmu\transpose \gamma_0 \\
& = \mv \left(\ma \Lambda \ma\transpose + \epsilon \mi_d \right)^{-1/2} \ms \mmu\transpose \gamma_0 \\
\implies \x_1 = -\etak \mv \left(\ma \Lambda \ma\transpose + \epsilon \mi_d \right)^{-1/2} \ms \mmu\transpose \gamma_0
\end{align*}
We show that $\x_1$ lies in the span, i.e. $\x_1 = \mx\transpose \beta$ by solving for $\beta \in \mathbb{R}^{n}$.
\begin{align*}
-\etak \mv \left(\ma \Lambda \ma\transpose + \epsilon \mi_d \right)^{-1/2} \ms \mmu\transpose \gamma_0 & = \mv \ms \mmu\transpose \beta \\
\intertext{Multiplying both sides by the full rank matrix $\mv$,}
\left(\ma \Lambda \ma\transpose + \epsilon \mi_d \right)^{-1/2} \ms \mmu\transpose (- \etak \gamma_0) & = \ms \mmu\transpose \beta \\
\intertext{Multiplying by the pseudo-inverse of $\ms$,}
\ms^{\dagger} \left(\ma \Lambda \ma\transpose + \epsilon \mi_d \right)^{-1/2} \ms \mmu\transpose (- \etak \gamma_0) & = \ms^{\dagger} \ms \mmu\transpose \beta \\
\intertext{Denoting $\ms^{\dagger} \ms \mmu\transpose$ as $\mz$, and multiplying both sides $\mz\transpose$,}
\mz\transpose \ms^{\dagger} \left(\ma \Lambda \ma\transpose + \epsilon \mi_d \right)^{-1/2} \ms \mmu\transpose (- \etak \gamma_0) & = \mz\transpose \mz \beta \\
\intertext{Multiplying by the inverse of the full-rank $n \times n$ matrix $\mz\transpose \mz$,}
\implies \beta = -\etak (\mz\transpose \mz)^{-1} \mz\transpose \ms^{\dagger} \left(\ma \Lambda \ma\transpose + \epsilon \mi_d \right)^{-1/2} \ms \mmu\transpose (\gamma_0)
\end{align*}
This implies that the iterate $\x_1$ lies in the span of the data points. For $\x_2$, $\left(\sum_{j = 0}^{1} \grad{\x_j} \grad{\x_j}\transpose + \epsilon \mi_d \right)^{-1/2}$ lies in the span of the iterates by a similar argument on the sum of terms. 

This implies that $\winf$, the iterate that interpolates the data also lies in the span of the iterates.
\end{proof}

\subsection{Generalization bounds}
\label{app:gen-bounds}
\begin{lemma}
The excess risk for the solution $\vw_{PGD}$ of PGD with a preconditioner $P$ and initialized at $\wn$ can be bounded as:
\begin{align}
\label{eq:risk-lin-reg}
R(\vw_{PGD}) & \coloneqq \E_{\vx, \vep} [\vx\transpose (\vw_{PGD} - \wstar) ]^2 \leq (\vw_0 - \wstar)\transpose \mb_\mxp\transpose \mathbf{\Sigma} \mb_\mxp(\vw_0 - \wstar) + \sigma^2 tr(\mc_\mxp)     
\end{align}
where $\mathbf{\Sigma} = \ex{\vx \vx\transpose}$ is the covariance matrix s.t. $\vx \sim N(0, \mathbf{\Sigma})$ and $\mb_\mxp = \mi - \mxp (\mx\transpose(\mx \mxp \mx\transpose)^{-1}) \mx$ and $\mc_\mxp= (\mx \mxp \mx\transpose)^{-1} \mx \mxp\transpose \, \mathbf{\Sigma} \, \mxp  \mx \transpose (\mx \mxp \mx\transpose)^{-1}$. 
\label{lemma:risk}
\end{lemma}
\begin{proof}
Recall that $\winf = \vw_0 + \mxp \ma(\vy-\vy_0) =  \vw_0 + \mxp \mx\transpose(\mx \mxp \mx\transpose)^{-1} [\vy - \mx  \vw_0]$. Define matrix $\ma = \mx\transpose(\mx \mxp \mx\transpose)^{-1}$, implying $\winf =  \vw_0 + \mxp \ma [\vy - \mx  \vw_0] =  \vw_0 + \mxp \ma [\mx \xopt + \epsilon - \mx  \vw_0]$
\begin{align*}
R(\winf) & =\ex[x]{(\mx\transpose(\winf-\xopt))^2} \\ 
& =\ex[x]{\left(\mx\transpose ( \vw_0 + \mxp \ma [X \xopt + \epsilon - \mx  \vw_0] - \xopt )\right)^2}\\ 
& = \ex[x]{\left((\mx\transpose ( ( \mi - \mxp \ma \mx ) \, ( \vw_0 -\xopt) + \mx\transpose \,  \mxp \ma \epsilon\right)^2}\\
& \leq 2 \ex[x]{\left(\mx\transpose ( ( \mi - \mxp \ma \mx ) \, ( \vw_0 -\xopt)\right)^2} + 2 \ex[x]{\left(\mx\transpose \,  \mxp \ma\epsilon \right)^{2}} \\
\intertext{Using the fact that $\covar = \ex{\vx \vx\transpose}$, $\mb = \mi - \mxp \ma \mx$ and $\mc= \ma\transpose P\transpose \, \covar \, \mxp \ma$.}
& = ( \vw_0 -\xopt)\transpose \mb\transpose\covar \mb (\wn - \xopt) + 2 \epsilon\transpose \mc \epsilon
\end{align*}
Take expectation w.r.t. the noise: 
\begin{align*}
\ex[\epsilon]{R(\winf)} & \leq (\wn -\xopt)\transpose \mb\transpose\covar \mb (\wn - \xopt) + \ex[\epsilon]{\epsilon\transpose \mc \epsilon} = (\wn -\xopt)\transpose \mb\transpose\covar \mb (\wn - \xopt) + \sigma^2 tr(\mc)
\end{align*}
since the noise has mean zero and variance $\sigma^2$. 
\end{proof}

\subsection{Finding an Optimal Preconditioner}
 It is clear that if we do not have additional information in the form of a validation set, remaining in the span of the points is the optimal strategy for an optimizer and the min-norm solution results in the best generalization for a general $\wstar$ and $\mathbf{\Sigma}$. 

Consequently, we consider a semi-supervised setting and investigate whether we can exploit unlabelled data and obtain better generalization. We consider an idealized case where we have infinite unlabelled data that enables us to get an accurate estimate of the true covariance matrix $\mathbf{\Sigma}$. We prove the following theorem, 
\begin{lemma}
\label{lemma:best-P}
Assume we have sufficient unlabeled data to estimate $\covar$ accurately. In the regression with noisy data we can leverage the unlabeled data to design an optimal preconditioner so as to reduce generalization risk. In the noiseless case, the unlabeled data doesn't help. 
\end{lemma}
This lemma shows that in the noiseless setting, having a preconditioner that keeps the update inside the data span can only help improving the convergence speed of the optimization regardless of the size of unlabeled data or the exact value of covariance matrix. However in the noisy case, we can find a preconditioner that has the minimum generalization error among all preconditioners including the identity matrix which gives us GD. Our experimental results for this section confirm our theoretical result.

\begin{proof}
We can decompose $\covar$ as follows: 
\begin{equation}
    \covar = \mx \transpose \ms_1 \mx + \mx \transpose \ms_2 \mx\prp + \mx \transpose\prp \ms_2\transpose \mx + \mx \transpose\prp \ms_3 \mx \prp, 
\end{equation}
and similarly we can decompose preconditioner:   
\begin{equation}
    \mxp = \mx \transpose \mxp_1 \mx + \mx\transpose \mxp_2 \mx \prp + \mx \transpose\prp \mxp_2\transpose \mx + \mx \transpose\prp \mxp_3 \mx \prp,  
\end{equation}
where $\covar , \mxp \in \mathbb{R}^{d\times d}, \ms_1,\mxp_1 \in \mathbb{R}^{n\times n },\ms_2,\mxp_2 \in \mathbb{R}^{n\times d-n }$, and $\ms_3,\mxp_3 \in \mathbb{R}^{d-n\times d-n }$. 
Recall the excess risk for linear regression 
\begin{align*}
    R(\winf^{opt})= \wstar{\transpose} (\mi -\mxp \mx {\transpose} \inv{(\mx \mxp \mx \transpose)}\mx )\transpose \covar  (\mi-\mxp \mx \transpose \inv{(\mx \mxp \mx{\transpose})}\mx ) \wstar.  
\end{align*}
To apply the above decompositions, note that we have 
\begin{align}
    & \mxp\mx\transpose = \mx \transpose \mxp_1 \mx \mx\transpose + \mx \transpose\prp \mxp_2\transpose \mx \mx\transpose \label{eq:dec1}\\
    & \inv{(\mx \mxp \mx \transpose)} = \inv{(\mx  \mx\transpose)}\inv{\mxp_1}\inv{(\mx \mx \transpose)}\label{eq:dec3}\\
    & \mxp \mx \transpose \inv{(\mx \mxp \mx\transpose)} = \mx \transpose \inv{(\mx \mx\transpose)} + \mx \prp \transpose \mxp_2\transpose \inv{\mxp_1} \inv{(\mx \mx \transpose)}\\
    & \mb = \mi-\mxp \mx \transpose \inv{(\mx \mxp \mx\transpose)}\mx = \mi - \mx\transpose \inv{(\mx \mx \transpose)}\mx - \mx\prp \transpose \mxp_2\transpose \inv{\mxp_1} \inv{(\mx \mx \transpose)}\mx\\
    & \mq = \mxp_2\transpose \inv{\mxp_1}\\
    & \mf= \mx\prp \transpose \mxp_2\transpose \inv{\mxp_1}  \inv{(\mx \mx \transpose)}\mx\\
    & \covar \mb = \covar - \mx \transpose \ms_1 \mx - \mx\prp \transpose \ms_2\transpose \mx - \mx \transpose \ms_2 \mx\prp\mf - \mx\prp\transpose \ms_3 \mx\prp\transpose \mf \\
    & \quad = \mx\transpose \ms_2 \mx\prp + \mx\prp\transpose \ms_3 \mx\prp - \mx \transpose \ms_2 \mf - \mx\prp\transpose \ms_3 \mx\prp \mf\\
    & \mb\transpose \covar \mb = (\mi - \mx\transpose \inv{(\mx \mx\transpose)}\mx -\mf\transpose)\covar \mb \\
    & \quad = \mx\transpose \ms_2 \mx\prp + \mx\prp\transpose \ms_3 \mx\prp - \mx \transpose \ms_2 \mf - \mx\prp\transpose \ms_3 \mx\prp \mf \\
    & \quad \quad - \mx\transpose \ms_2 \mx\prp + \mx\transpose\ms_2 \mf \\
    & \quad \quad -\mf\transpose \mx\transpose \ms_2 \mx\prp - \mf\transpose \mx\prp\transpose \ms_3 \mx\prp + \mf\transpose \mx\transpose \ms_2 \mx\prp \mf + \mf\transpose \mx\prp\transpose \ms_3 \mx\prp \mf \\
    & \quad = \mx\prp \transpose \ms_3 \mx\prp -\mx\prp \transpose \ms_3 \mx\prp \mf - \mf\transpose \mx\prp\transpose \ms_3 \mx\prp + \mf\transpose \mx\prp\transpose \ms_3 \mx\prp \mf \\
    & tr(\mf\transpose \mx\prp\transpose \ms_3 \mx\prp) = 0 \\
    & tr(\mb\transpose \covar\mb) = tr( \mx\prp \transpose \ms_3 \mx\prp ) + tr (\mf\transpose \mx\prp\transpose \ms_3 \mx\prp \mf) = cnst +  tr (\mf\transpose \mx\prp\transpose \ms_3 \mx\prp \mf )
\end{align}
In this case we see that the best trace is achieved when we set $\mxp_2 = 0$ which is independent of the information of $\covar$. 
Now assume noisy case whose excess risk has extra term which is: 
\begin{equation}
   tr(\mc) =  \sigma^2 tr(\inv{(\mx \mxp \mx\transpose)}\mx \mxp \covar \mxp \mx\transpose \inv{(\mx \mxp \mx\transpose)}). 
\end{equation}
Now we expand this extra term based on the aforementioned
decompositions
\begin{align}
     \covar  \mxp \mx\transpose \inv{(\mx \mxp \mx\transpose)} &= (\mx \transpose \ms_1 \mx + \mx \transpose \ms_2 \mx\prp + \mx \transpose\prp \ms_2\transpose \mx + \mx \transpose\prp \ms_3 \mx \prp)
    (\mx \transpose \inv{(\mx \mx\transpose)} + \mx \prp \transpose \mxp_2\transpose \inv{\mxp_1} \inv{(\mx \mx \transpose)})\\
    & = \mx\transpose \ms_1+ \mx\transpose \ms_2 \mx\prp \mx\prp\transpose \mq \inv{(\mx \mx\transpose)}+ \mx\prp\transpose \ms_2\transpose + \mx\prp\transpose \ms_3 \mx\prp \mx\prp\transpose \mq \inv{(\mx \mx\transpose)}\\
   \inv{(\mx \mxp \mx\transpose)}\mx \mxp \covar  \mxp \mx\transpose \inv{(\mx \mxp \mx\transpose)} &= \ms_1 + \ms_2 \mx\prp \mx\prp\transpose \mq  \inv{(\mx \mx\transpose)}\\
   & +  \inv{(\mx \mx\transpose)} \mq\transpose  \mx\prp \mx\prp\transpose \ms_2\transpose + \inv{(\mx \mx\transpose)} \mq\transpose \mx\prp \mx\prp\transpose \ms_3 \mx\prp \mx\prp\transpose \mq \inv{(\mx \mx\transpose)}
\end{align}
So the the excess risk is upperbounded by 
\begin{align}
    tr(\mb\transpose \covar \mb) + \sigma^2 tr(\mc) &= tr( \mx\prp \transpose \ms_3 \mx\prp ) +  tr (\mf\transpose \mx\prp\transpose \ms_3 \mx\prp \mf )\\
    &+ \sigma^2 (\ms_1 + \ms_2 \mx\prp \mx\prp\transpose \mq  \inv{(\mx \mx\transpose)}
    +  \inv{(\mx \mx\transpose)} \mq\transpose  \mx\prp \mx\prp\transpose \ms_2\transpose + \inv{(\mx \mx\transpose)} \mq\transpose \mx\prp \mx\prp\transpose \ms_3 \mx\prp \mx\prp\transpose \mq \inv{(\mx \mx\transpose)})
\end{align}
Now if we take derivative w.r.t. $\mq$ and set it to zero we get
\begin{equation}
    \mq^* =- \inv{(\mk\prp \ms_3 \mk\prp )}(\sigma^2 \mk\prp \ms_2\transpose \inv{\mk}) \inv{(\inv{\mk} + \mk^{-2})}
\end{equation}
where $\mk = \mx \mx\transpose$ and $\mk\prp = \mx\prp \mx\prp \transpose$. By setting $\mxp_1 = \mi$, we have $\mxp_2 = \mq^{*}{\transpose}$. Therefore an optimal $\mxp$ is 
\begin{equation}
\mxp = \mx\transpose \mx + \mx\transpose \mq^{*}{\transpose} \mx\prp +  \mx{\prp}{\transpose} \mq^{*} \mxp   
\end{equation}
Note that here we assume we have enough unlabeled date which among them we can pick $n-d$ of them which are orthogonal to the training data i.e. 
\begin{equation}
    \mx\transpose\mx\prp=\mx\prp\transpose\mx = 0 
\end{equation}

\end{proof}

\section{Additional proofs for linear classification}
\label{app:proofs-lin-class}

\subsection{Counter-examples for squared-hinge loss}
\label{app:sq-hinge-counter}
\begin{example}
\label{exmp:exmp1}
Consider two points $\vx_1 =(-1,0)$ and $\vx_2 =(a,b=\sqrt{1-a^2})$ where $0< a < 1$ with $y_1 = -1$ and $y_2=1$ as labels respectively. It can be shown that $\vws = (1,\frac{b}{1+a})$ is the min-norm solution and we have $y_1\dpr{\vws}{\vx_1} = y_2\dpr{\vws}{\vx_2} =1$. In the following we assume three different cases where gradient descent starting from $\vw_0$ over squared hinge loss $\mathcal{L} (\vw) = \frac{1}{4}\sum_{i=1}^2 \left(\max\{0,1 - y_i \langle \vw, \vx_i \rangle\} \right)^{2}$ with any fixed step size $\eta > 0$ won't converge to $\vws$. 
\begin{enumerate}
    \item let $\vw^0= \vws + (\alpha,\beta)$ where $\alpha, \beta \geq 0 $ and $\alpha + \beta > 0$. Since $y_1\dpr{\vw_0}{\vx_1} \geq 1$ and $y_2\dpr{\vw_0}{\vx_2} \geq 1$, therefore the squared hinge loss is zero and so is its gradient. Therefore GD would not progress and $\vw_0$ is an answer which is not min-norm solution. 
    
    \item Now assume $\vw_0$ classifies $\vx_2$ with margin bigger than one and $\vx_1$ with margin smaller than one i.e. $y_2\dpr{\vw_0}{\vx_2} \geq 1$ and $y_1\dpr{\vw_0}{\vx_1} < 1$. In this situation the gradient of loss function at $\vw_0$ 
    \begin{align*}
        \nabla \mathcal{L}(\vw_0) = -0.5 y_1\vx_1(1-y_1\vx_1\transpose \vw_0) = -0.5 \alpha_0 y_1\vx_1.
    \end{align*}
    Note that $ \nabla \mathcal{L}(\vw_0)=(-0.5 \alpha_0, 0)$ and $\alpha_0 > 0$. Now if we run GD for one step we have 
    \begin{align*}
        \vw_1= \vw_0 + ( 0.5 \eta \alpha_0, 0 ). 
    \end{align*}
    Observe that $y_2\dpr{\vw_1}{\vx_2} = y_2\dpr{\vw_0}{\vx_2} + 0.5 \eta \alpha_0 a >= 1$ since $0.5 \eta \alpha_0 a \geq 0$. Therefore the loss value at $\vw_2$ for $\vx_2$ is also zero. Therefore for all $\vw_t$ we have $y_2\dpr{\vw_t}{\vx_2} \geq 1$ and loss functions gradient at any $\vw_t$ just add some positive value to first component of $\vw_t$. Therefore if we assume GD converges to $\vw_{\infty}=(\vw_{\infty}^1,\vw_{\infty}^2)$, we have $\vw_{\infty}^2=\vw_0^2\neq \vws{^2}$. 
    
    \item Here we consider the reverse of the above scenario i.e.  $y_2\dpr{\vw_0}{\vx_2} < 1$ and $y_1\dpr{\vw_0}{\vx_1} = \vw_0^1 \geq 1$. 
    \begin{align*}
        &\nabla \mathcal{L}(\vw_0) = -0.5 \vx_2(1-y_2\vx_2\transpose \vw_0) = -0.5 \alpha_0 \vx_2\\
        &\vw_1= \vw_0 + 0.5 \eta \alpha_0 \vx_2. 
    \end{align*}
    We can check that $y_1\dpr{\vw_1}{\vx_1} = y_1\dpr{\vw_0}{\vx_1} + y_1\dpr{\vx_2}{\vx_1} \geq 1 + 0.5 \eta \alpha_0 a \geq 1$ since $a, \alpha_0 > 0$. Therefore similar to above scenario, for all $t$ we have $y_1 \dpr{\vw_t}{\vx_1} \geq 1$. We can observe that $\nabla \mathcal{L} (\vw_t) = \gamma \vx_2$ for some $\gamma > 0$. Therefore $\vw_{\infty} = \vw_0+\beta \vx_2 = (\vw_0^1 + \beta a , \vw_0^2 + \beta b)$ for some $\beta > 0$ that means $\vw_{\infty}^1 = \vw_0^{1} + \beta a \neq 1 = \vws$.
\end{enumerate}
\end{example}

\begin{example}
\label{exmp:exmp2}
Consider the same dataset as above, and consider GD initialized with $\vw_0=0$ vector. The GD update at step $t$ is 
\begin{align*}
    \vw_{t+1} = \vw_t -\eta \nabla \mathcal{L}(\vw_t). 
\end{align*}
Based on case 2 and case 3 of Example~\ref{exmp:exmp1}, if for any $t$ we have $y_1\dpr{\vw_t}{\vx_1} \geq 1$ or $y_2\dpr{\vw_t}{\vx_2} \geq 1$ then we know that GD won't converge to $\vws$. Now assume that forall $t<\infty$ we have $y_1\dpr{\vw_t}{\vx_1} < 1$ or $y_2\dpr{\vw_t}{\vx_2} < 1$. To make notation simpler let assume $\vx_1 = y_1\vx_1$ and $\vx_2 = y_2\vx_2$. In this case the update rule for GD is 
\begin{align*}
    \vw_{t+1}&=\vw_t +\frac{\eta}{2}(\vx_1+\vx_2) - \frac{\eta}{2}(\vx_1\vx_1\transpose + \vx_2\vx_2\transpose)\vw_t = (\mi - \frac{\eta}{2}\mx\transpose\mx)\vw_t +\frac{\eta}{2}(\vx_1+\vx_2)\\
    & = \ma^t \vw_0 + \frac{\eta}{2}(\vx_1+\vx_2)\sum_{i=0}^{t} \ma^i =\frac{\eta}{2}(\vx_1+\vx_2)\sum_{i=0}^{t} \ma^i . 
\end{align*}
It can be seen that the eigenvalues of $\mx\transpose \mx$ are $\lambda_1= 1+ a$ and $\lambda_2= 1-a$. To get convergence at $t \rightarrow \infty$, we need $\sum_{i=0}^{\infty} \ma^i$ to be  Neumann series. To get that, we need to $\eta \leq \frac{\alpha}{\lambda_1}$ where $\alpha < 1$. The largest eigenvalue of $\ma$ is $1-\alpha\frac{\lambda_2}{\lambda_1}$. Hence we have 
\begin{align*}
    \vw_{\infty} = \frac{\eta}{2}(\vx_1+\vx_2)\sum_{i=0}^{\infty} \ma^i = \frac{\eta}{2}(\vx_1+\vx_2)\alpha \frac{\lambda_1}{\lambda_2}= \frac{\alpha^2}{1-a}(\vx_1+\vx_2)=\frac{\alpha^2}{2(1-a)}(1+a, b). 
\end{align*}
To get convergence to $\vws$ i.e. $\vw_{\infty}=\vws$, we need that $\frac{\alpha^2(1+a)}{2(1-a)}=1$ and $\frac{\alpha^2 b}{2(1-a)}=\frac{b}{1+a}$. Therefore we have to set $\alpha = \sqrt{\frac{2(1-a)}{1+a}}$. However if we pick $a \le \frac{1}{3}$ then $ \alpha \geq 1$ which is invalid value for $\alpha$. 
\end{example}

\newpage
\section{Experiments for over-parameterized linear regression}
\label{app:exp-regression}

\subsection{Additional Results}
\label{app:regression_additional_results}
This section presents further experimental results for over-parameterized linear regression problems. Appendix~\ref{app:additional_ntk_regression} extends our investigation of regression with neural tangent kernels to several real-world datasets from the UCI repository~\cite{Dua:2019}. For completeness, we also repeat the experiment from Figure~\ref{fig:ntk_synthetic_regression} with batch (deterministic) gradients. Then, in Appendix~\ref{app:regression_verification}, we experimentally verify the theoretical results discussed in Section~\ref{sec:lin-reg}, including Lemmas~\ref{lemma:unique-span-gen}, \ref{lemma:adagrad}, \ref{lemma:newton}, and \ref{lemma:opt-P-reduction}. Finally, Appendix~\ref{app:best-p} investigates improving the generalization of preconditioned gradient-descent by optimizing over the space of preconditioners. In particular, we consider optimizing the excess-risk bound in Lemma~\ref{lemma:risk} as well as several simplified upper-bounds based on this quantity.

\subsubsection{Regression with Neural Tangent Kernels}
\label{app:additional_ntk_regression}

We investigate the implicit regularization of Adagrad and SGD for kernel regression on the \texttt{mushroom} and \texttt{wine} datasets. As in our synthetic experiments, we fit the model via the squared loss and use features from the neural tangent kernel of single-layer, feed-forward networks with $50$ and $100$ hidden units, respectively. Unlike  Figure~\ref{fig:ntk_synthetic_regression}, we also consider variants of Adagrad where the model is projected on the data span after each iteration. We show only test loss for these projected optimizers, since training performance is unaffected by projecting onto the data span\footnote{Projection operators with a large condition number do introduce a precision floor on the training loss.}. 

Results are shown in Figure~\ref{fig:ntk_mushroom_wine}. The generalization performance for Adagrad shows a striking dependence on step-size --- large step-sizes obtain test loss approximately two orders of magnitude larger than the smallest considered --- while training loss is largely unaffected. Similarly to our synthetic experiments, SGD stalls on the training loss, but still generalizes well. Of particular interest are the projected versions of Adagrad, which completely correct for the poor generalization performance of the "vanilla" algorithm and obtain a test loss comparable to SGD.

As an ablation, we repeat the synthetic regression NTK experiment from the main paper (Figure~\ref{fig:ntk_synthetic_regression}) with batch gradients. We also include projected variants of Adagrad following the protocol above. Figure~\ref{fig:ntk_synthetic_deterministic} shows that the trends from the stochastic setting also hold for deterministic optimization; Adagrad converges quickly in comparison to tuned gradient descent, which stalls on the ill-conditioned problem. Yet, Adagrad's test performance depends strongly on the step-size chosen and never out-performs the min-norm solution. The projected variants of Adagard correct for this poor generalization and also converge faster than gradient descent.

\begin{figure}
    \centering
    \includegraphics[width=0.9\textwidth]{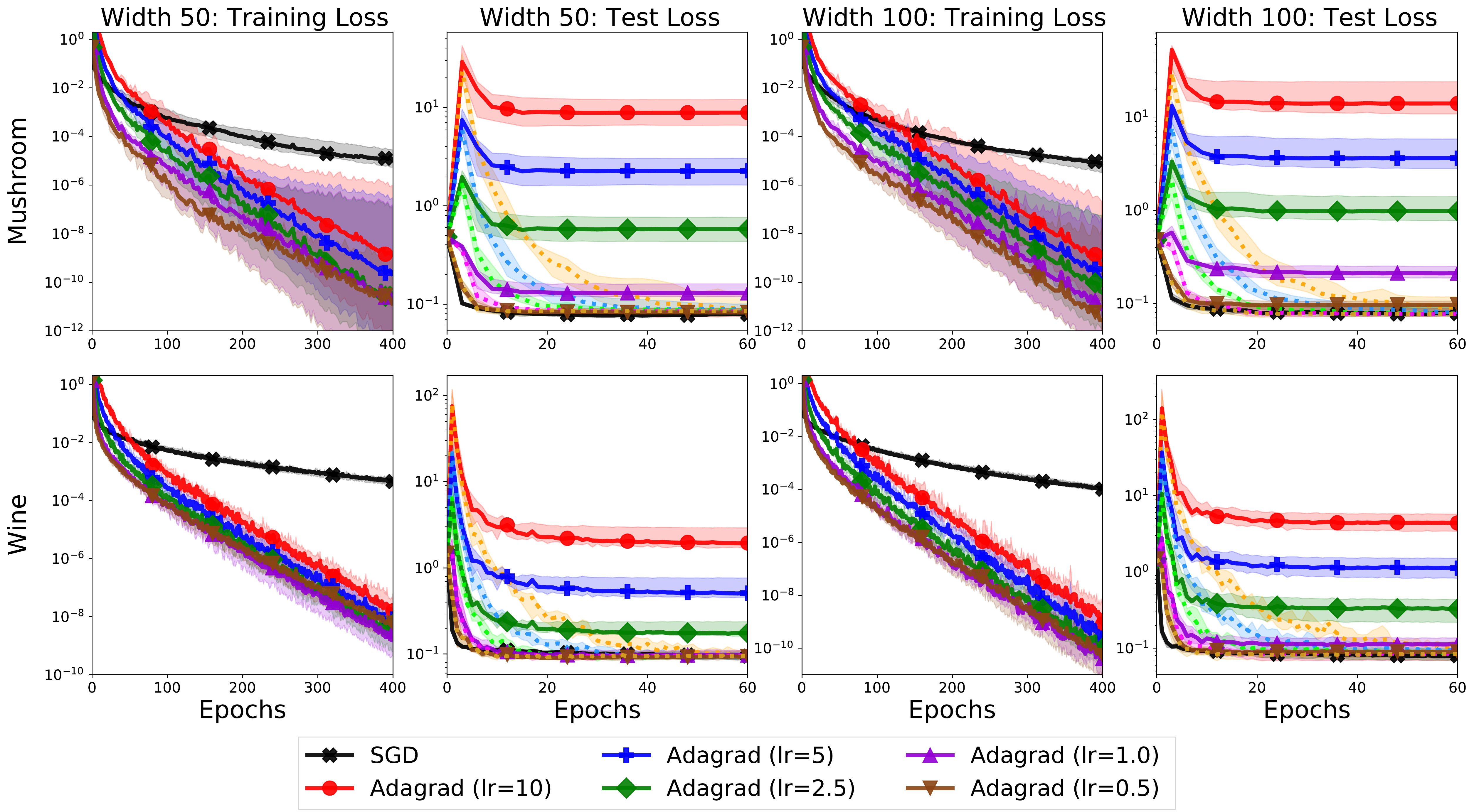}
    \caption{Performance of SGD and Adagrad for regression on the \texttt{mushroom} and \texttt{wine} datasets using squared-loss and the NTK of one-layer networks with 50 and 100 hidden units. We use the largest step-size for which SGD converges and consider a range of step-sizes for Adagrad. For each step-size, we plot as a dotted line the test loss (training loss is unchanged) for a variant of Adagrad where the model parameters are projected onto the span of the training data after every iteration. Tuned SGD stalls on the training loss but generalizes well, while Adagrad's generalization depends strongly on step-size. Projecting onto the data span corrects this behavior.  }
    \label{fig:ntk_mushroom_wine}
\end{figure}

\begin{figure}
    \centering
    \includegraphics[width=0.9\textwidth]{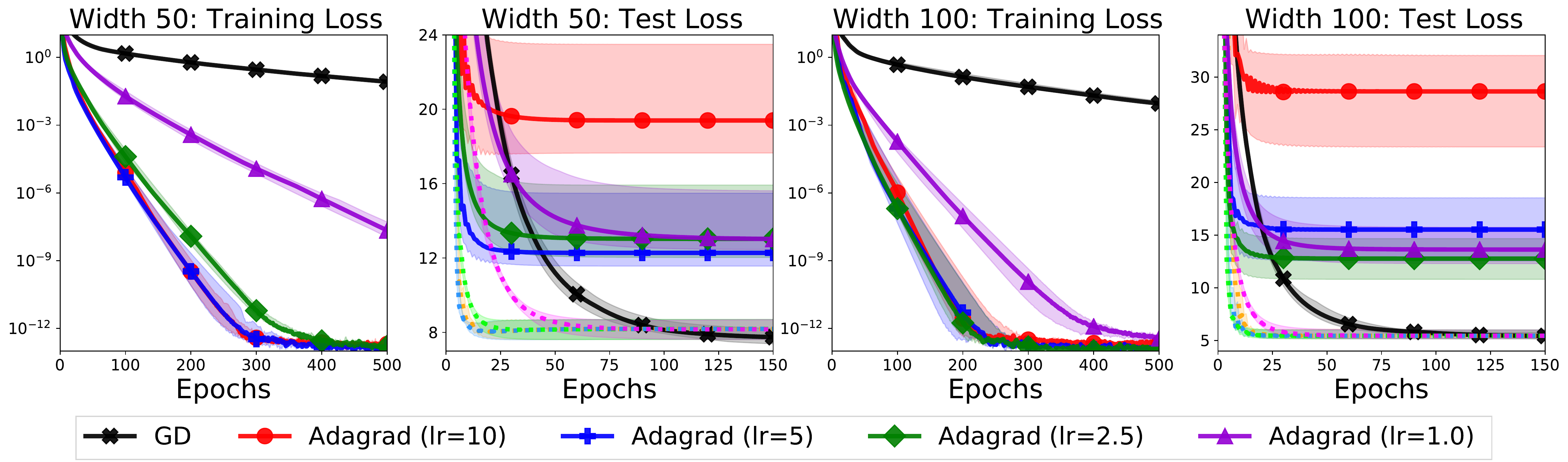}
    \caption{Ablation for synthetic regression problem using NTK features. Unlike in Figure~\ref{fig:ntk_synthetic_regression}, we consider batch optimization. Moreover, for each step-size, we plot as a dotted line the test loss (training loss is unchanged) for a variant of Adagrad where the model parameters are projected onto the span of the training data after every iteration. We see that deterministic optimization does not change Adagrad's varied generalization performance. Projecting onto the data span ensures comparable generalization to GD. }
    \label{fig:ntk_synthetic_deterministic}
\end{figure}

\subsubsection{Verification of Theoretical Results}
\label{app:regression_verification}

Now we verify our theoretical results for over-parameterized regression with several experiments on synthetic problems.

\textbf{Lemma~\ref{lemma:unique-span-gen}}: Figure~\ref{fig:unique-span-gen} examines the convergence of PGD with randomly generated preconditioners in the mini-batch and batch settings. For each generated preconditioner $\mxp$, we compute the solution to the normal equations
\[ w^*_p  = \mxp \mx\transpose (\mx \mxp \mx\transpose)^{-1} \vy, \]
and plot convergence of the iterates generated by PGD (with $\mxp$) to the solution $w^*_p$.
The convergence of gradient descent to the min-norm solution is shown as a baseline. 
We clearly see that PGD converges to the $\mxp^{-1}$-norm least squares solution as established by Lemma~\ref{lemma:unique-span-gen}.

\begin{figure}
    \centering
    \includegraphics[width=0.9\textwidth]{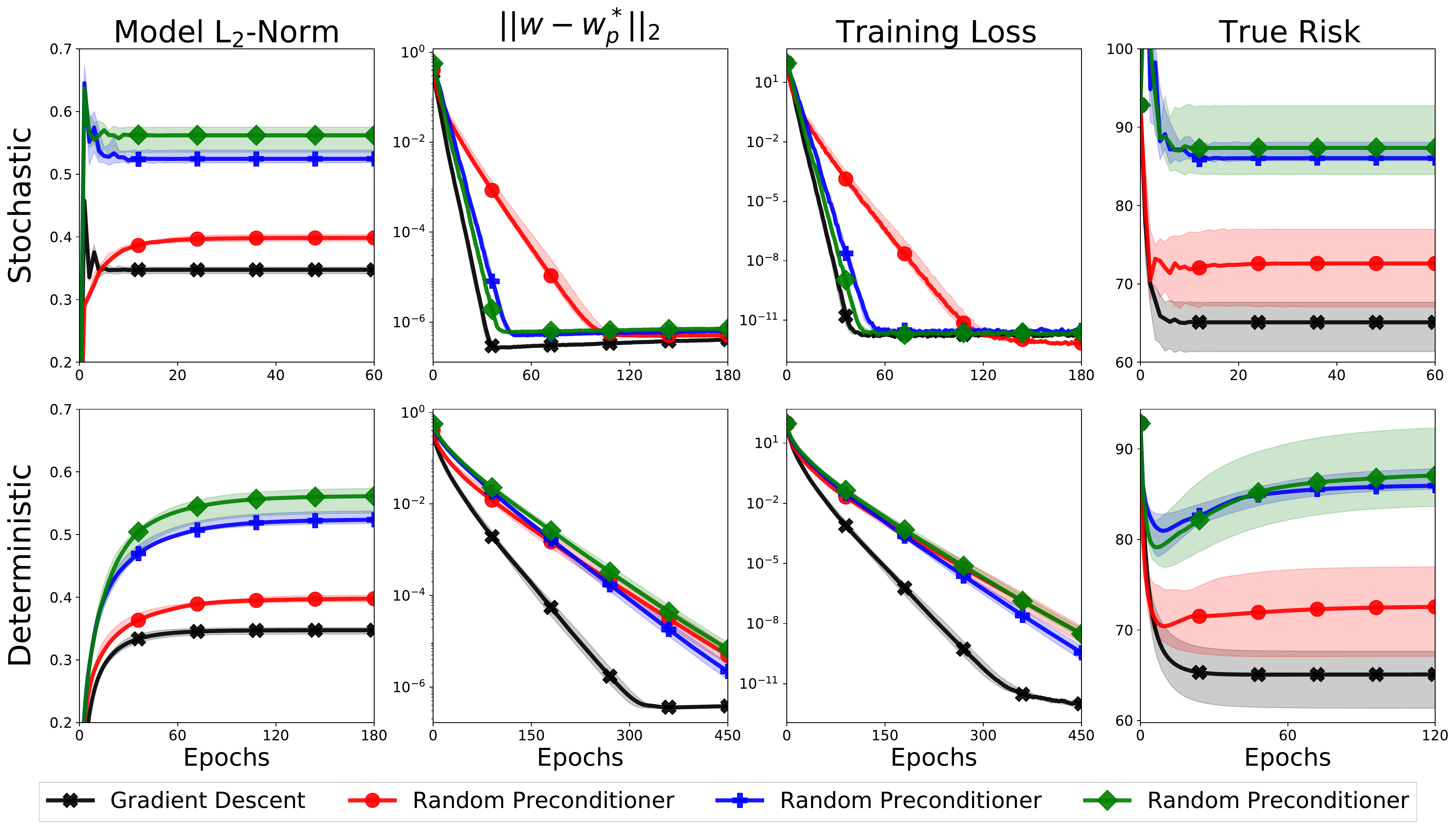}
    \caption{Experimental validation of Lemma~\ref{lemma:unique-span-gen}. Each PGD method uses a randomly generated diagonal preconditioner $\mxp$. The metric $\norm{w - w_P}$ is the $\ell_2$ distance of the current weight vector $w$ to $w_p$ -- the interpolating solution with minimum $\norm{\cdot}_{\mxp^{-1}}$ norm. Each PGD method converges to the min-norm solution in its preconditioner $\mxp^{-1}$, as predicted. }
    \label{fig:unique-span-gen}
\end{figure}

\textbf{Lemmas~\ref{lemma:newton} and \ref{lemma:adagrad}}:  Figure~\ref{fig:lemmas_newton_adagrad} shows convergence of Newton's method and full-matrix Adagrad for a synthetic regression problem. Unlike all other experiments with synthetic regression data, we generate a well-conditioned dataset to avoid complications with evaluating the Hessian. We see that Newton's method and full-matrix Adagrad remain in the span of the data and converge to the min-norm solution as predicted.

\begin{figure}
    \centering
    \includegraphics[width=0.9\textwidth]{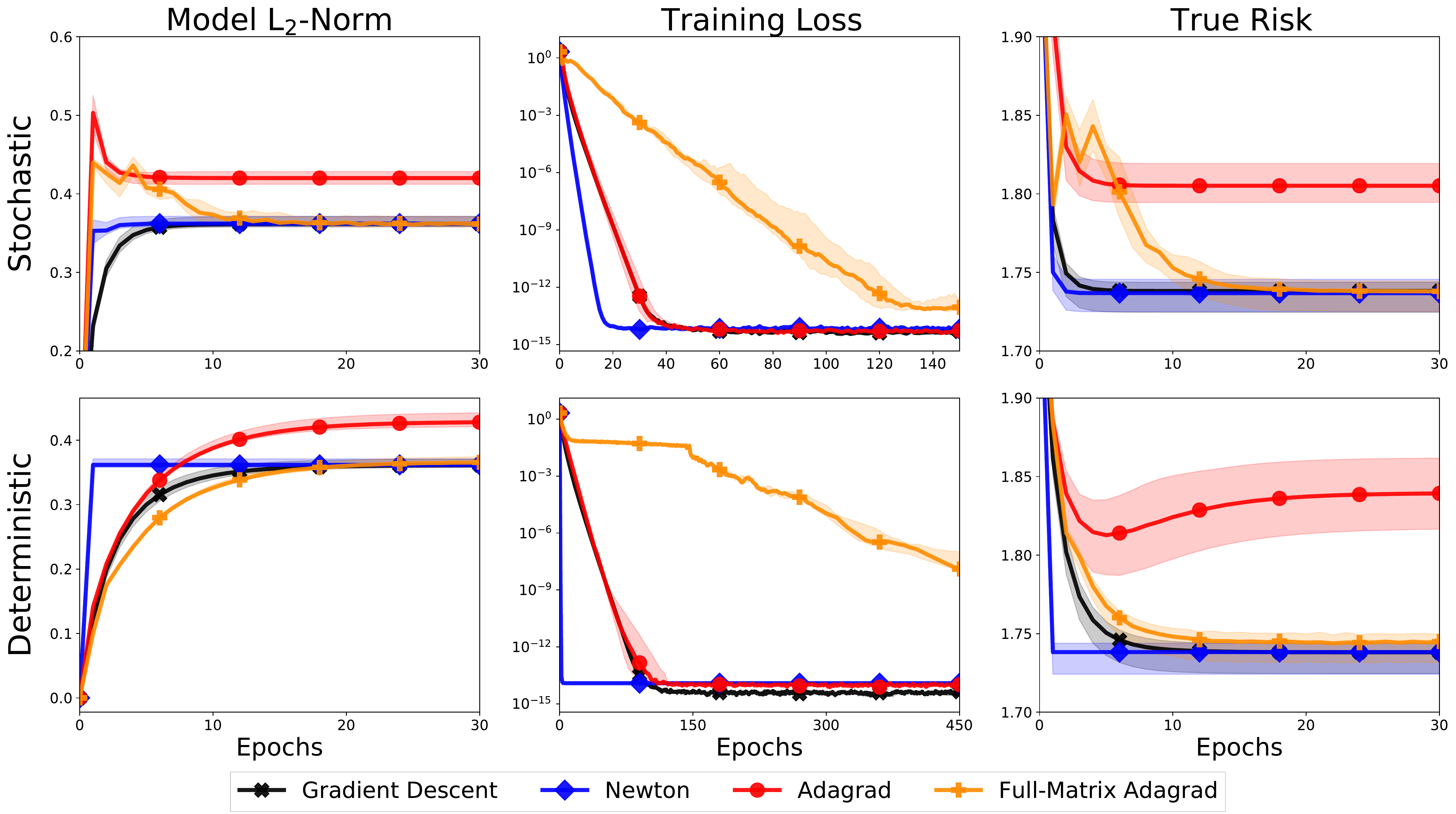}
    \caption{Experimental validation of Lemmas~\ref{lemma:newton} and \ref{lemma:adagrad} with batch and mini-batch gradients. Newton's method and full-matrix Adagrad converge to the min-norm solution, while diagonal Adagrad does not. }
    \label{fig:lemmas_newton_adagrad}
\end{figure}

\textbf{Lemma~\ref{lemma:opt-P-reduction}}: Now we experimentally confirm that every interpolating solution $\wopt$ has a corresponding preconditioner $\mxp$ for which PGD with $\mxp$ converges to $\wopt$. We run the Adam~\cite{kingma2014adam}, Adagrad~\cite{duchi2011adaptive}, and Coin Betting~\cite{orabona2017training} optimizers until convergence and then construct the corresponding preconditioners. Figure~\ref{fig:opt-P-reduction} shows the optimization and generalization performance of both the original optimizers and their associated PGD methods (dashed lines). The PGD methods converge to interpolating models with the same $\ell_2$ norm and true risk as the original optimizers and show similar final training loss.

\begin{figure}
    \centering
    \includegraphics[width=0.9\textwidth]{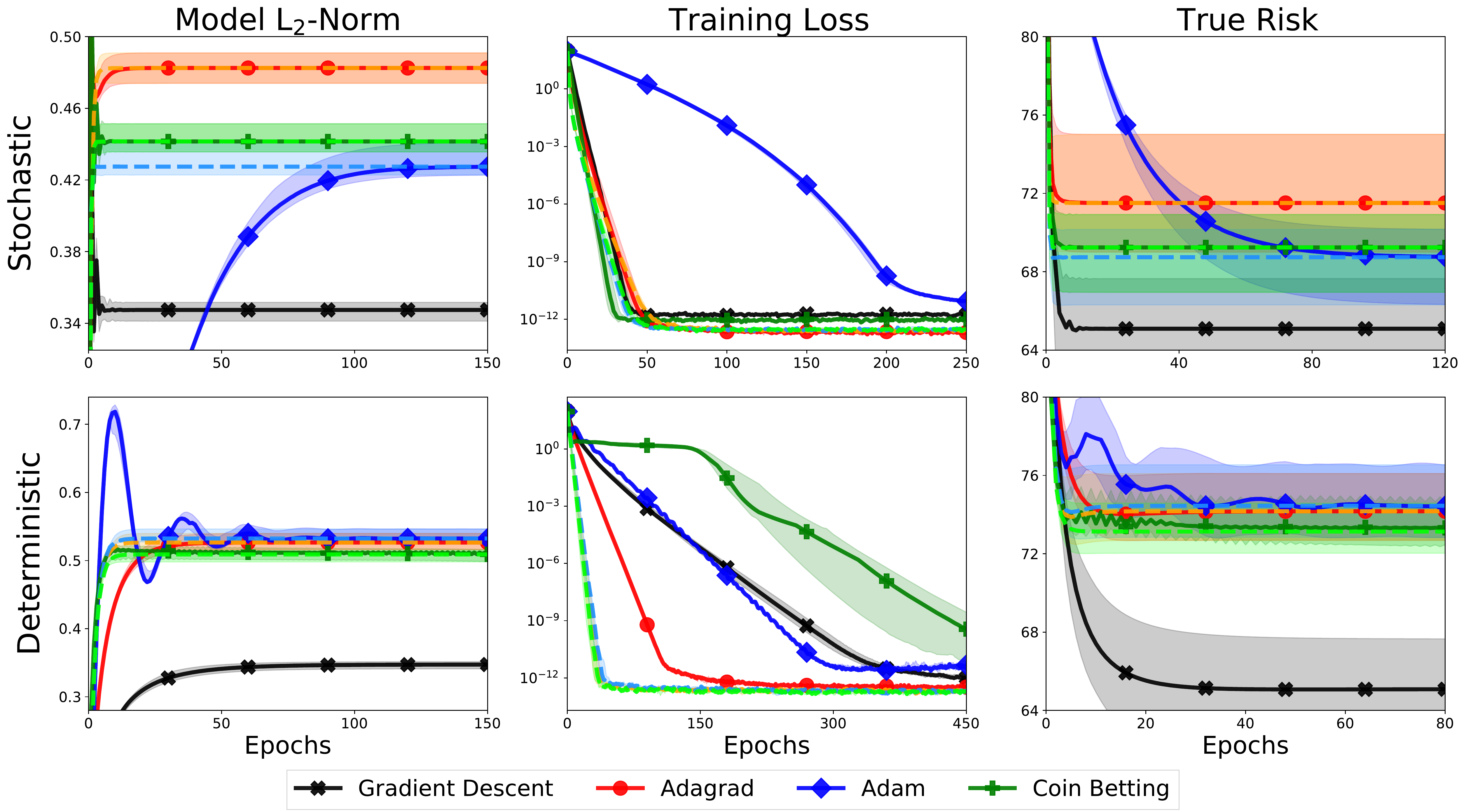}
    \caption{Experimental validation of Lemma~\ref{lemma:opt-P-reduction}. For each optimizer, the corresponding dashed line shows the convergence of PGD with a preconditioner constructed from the interpolating solution found by that optimizer as in Lemma~\ref{lemma:opt-P-reduction}. The three PGD methods converge to the same solutions as the original optimizers as predicted. }
    \label{fig:opt-P-reduction}
\end{figure}

\textbf{Proposition~\ref{prop:proj-span}}: Figure~\ref{fig:proj-span} explores the effects of projecting onto the span of the training data during optimization. We use the Adagrad, Adam, and Coin Betting optimizers and compare projecting the model parameters onto $\text{span}(\mx\transpose)$ at every iteration with the "default" algorithms. We make several observations: (i) the default algorithms show varied generalization performance and are all out-performed by GD, (ii) the projected variants converge to the min-norm solution and obtain test loss comparable to GD, and (iii) the improved generalization of the projected methods is consistent across the stochastic and deterministic cases.

\begin{figure}
    \centering
    \includegraphics[width=0.9\textwidth]{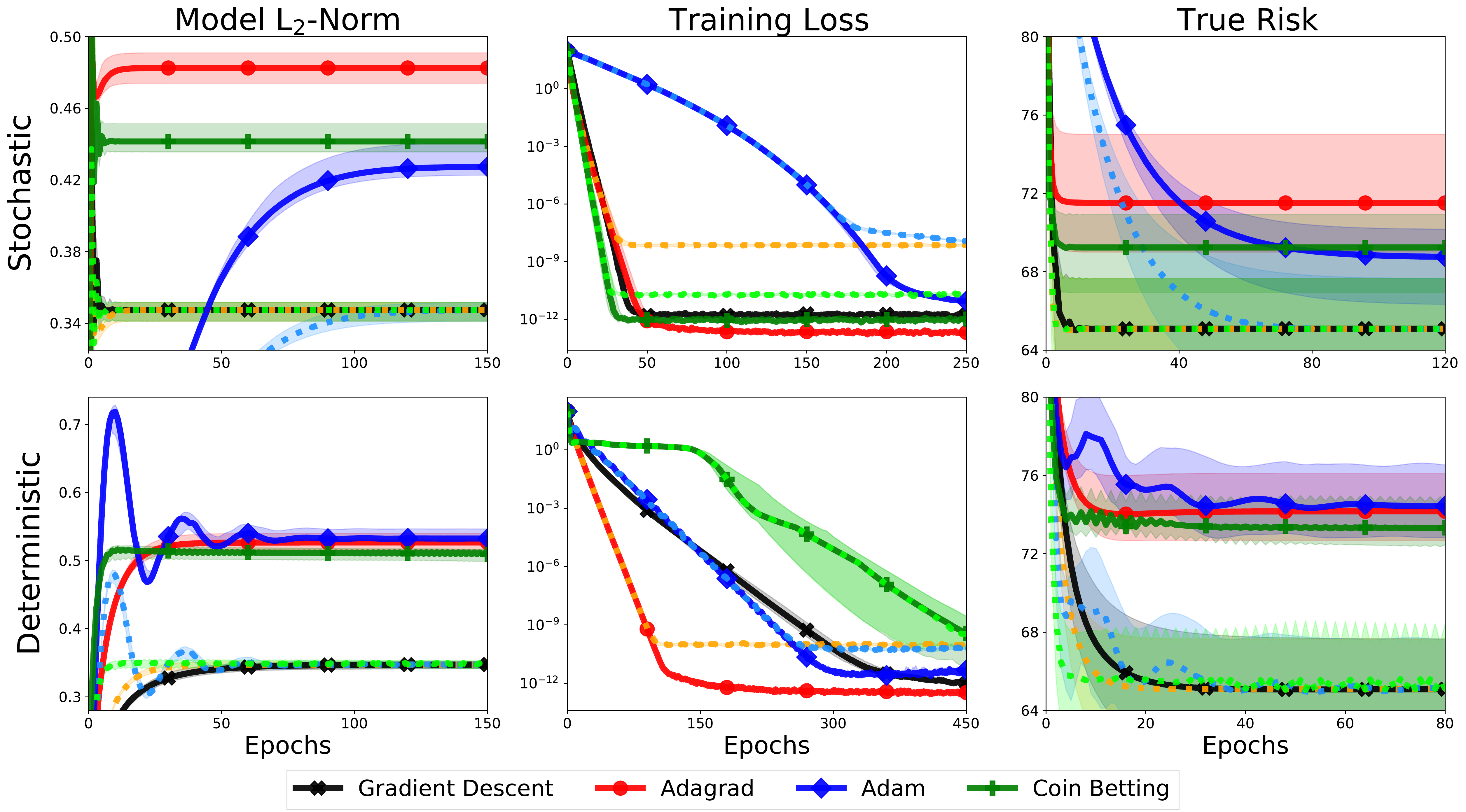}
    \caption{Effects of projecting onto $\text{span}(\mx\transpose)$ after every iteration when using the Adagrad, Adam, and Coin Betting optimizers. Solid lines with markers denote standard optimizers, while dotted lines shows the same algorithms with projections onto the span of the data at each iteration. We note that projecting onto the data span introduces a precision floor for the training loss, but otherwise does not affect optimization dynamics of the original optimizers. Projected methods converge to the min-norm solution and display similar generalization to gradient descent, as predicted by Proposition~\ref{prop:proj-span}. }
    \label{fig:proj-span}
\end{figure}

\subsubsection{Improved Generalization via Better Preconditioners}
\label{app:best-p}

In this section, we investigate choosing a preconditioner $\mxp$ to minimize the bound on the excess risk given in Lemma~\ref{lemma:risk}. We consider minimizing the exact bound as well as two upper-bounds on the excess-risk bound that do not require full knowledge of the true model $w^*$. Starting from Lemma~\ref{lemma:risk}, we have
\begin{align}
R(\vw_{PGD}) \coloneqq \E_{\vx, \vep} [\vx\transpose (\vw_{PGD} - \wstar) ]^2 &\leq (\vw_0 - \wstar)\transpose \mb_\mxp\transpose \mathbf{\Sigma} \mb_\mxp(\vw_0 - \wstar) + \sigma^2 tr(\mc_\mxp) & \tag{Exact} \nonumber\\
&\leq \norm{\vw_0 - \wstar} \norm{\mb_\mxp\transpose \mathbf{\Sigma} \mb_\mxp}_{2} + \sigma^2 tr(\mc_\mxp)~\label{eq:operator_risk_bound} & \tag{Operator} \\
&\leq \norm{\vw_0 - \wstar} \norm{\mb_\mxp\transpose \mathbf{\Sigma} \mb_\mxp}_{F} + \sigma^2 tr(\mc_\mxp).~\label{eq:frobenius_risk_bound} & \tag{Frobenius} 
\end{align}
When $\vw_0 = 0$, these upper-bounds only require knowledge of the norm of the true model, rather than the true model itself. They are particularly attractive in the noiseless case ($\sigma^2 = 0$), where no knowledge of $w^*$ is required and we need only optimize the two-operator norm or Frobenius norm of the matrix $\mb_\mxp\transpose \mathbf{\Sigma} \mb_\mxp$.

We optimize the exact excess risk bound ("Exact") and upper-bound using the Frobenius norm ("Frobenius") with respect to a diagonal preconditioner $\text{diag}(\mathbf{p})$ as well as a full-matrix preconditioner $\mxp$. The upper-bound using the 2-operator norm ("Operator") is used to learn a diagonal preconditioner only, since evaluating the maximum singular value is computationally expensive. We also consider using PGD with $\mathbf{\Sigma}^{-1}$ as a preconditioner ("Sigma"), which corresponds to natural gradient descent~\cite{amari1998natural}. Lemma~\ref{lemma:RC_Precond} shows that the empirical precision $\hat{\mathbf{\Sigma}}^{-1}$ minimizes an upper bound on the Rademacher complexity for the family of linear classifiers with bounded $\mxp^{-1}$ norm, which suggests that PGD with $\mathbf{\Sigma}^{-1}$ may generalize well. For completeness, we also compare with tuned stochastic gradient descent and Adagrad.

Figure~\ref{fig:optimizing_p} shows the results of optimizing over $P$ in both the noiseless setting ($\sigma^2 = 0$) and the case where $\sigma^2 = 1$. We observe that minimizing the exact bound on the excess risk is highly effective --- especially when using a full-matrix preconditioner. In this latter setting, PGD finds an interpolating solution with near optimal risk. In contrast, the preconditioners obtained by minimizing the upper-bounds given by the Frobenius and 2-operator norms yield interpolating solutions that generalize only as well as the min-norm solution. This suggests that these bounds are too loose to be useful for learning better preconditioners. Such a conclusion is corroborated by Table~\ref{table:optimized_p_risk_bounds}, which shows that the Operator and Frobenius preconditioners do not improve the excess-risk bound over the identity matrix. Lastly, it is highly interesting to note that while natural gradient descent converges very quickly, it obtains the worst generalization performance out of all methods considered. We speculate that this is because of the discrepancy in rank between the empirical and true covariance matrices: $\text{rank}(\mathbf{\hat \Sigma}) \leq n \ll \text{rank}(\mathbf{\Sigma}) = d$.

\begin{figure}
    \centering
    \includegraphics[width=0.9\textwidth]{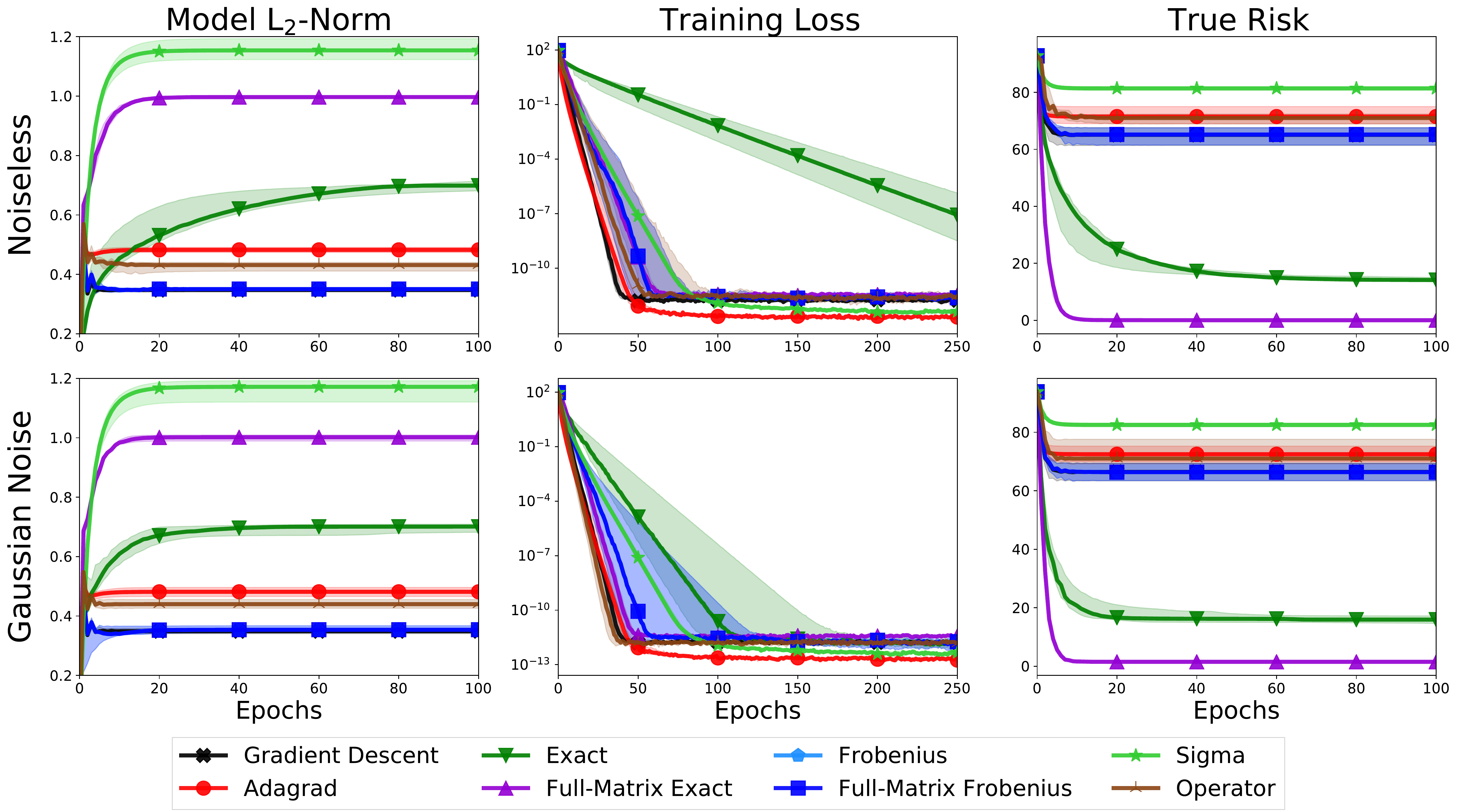}
    \caption{Generalization performance of stochastic PGD where the preconditioner is selected by optimizing bounds on the excess risk. The preconditioner for "Exact" minimizes the excess risk bound given in Lemma~\ref{lemma:risk}, while "Operator" and "Frobenius" minimize upper-bounds~\ref{eq:operator_risk_bound} and \ref{eq:frobenius_risk_bound}, respectively. "Sigma" uses the true precision matrix of the data, $\mathbf{\Sigma}^{-1}$, as the preconditioner. The full-matrix and diagonal preconditioners learned by minimizing the exact excess-risk bound greatly improve generalization performance over the min-norm solution, while "Frobenius" and "Operator" fail to outperform SGD and only minorly improve upon Adagrad. }
    \label{fig:optimizing_p}
\end{figure}


\begin{table}[]
    \centering
    \begin{tabular}{c|c c c c c c c}
    Noise Level & Identity & Exact & Exact (FM) & Frobenius & Frobenius (FM) & Operator & Sigma\\ \hline
    0 & $64.45 \pm 3.39$ & $14.13 \pm 1.07$ & $0.02 \pm 0.02$ & $64.47 \pm 3.37$ & $64.7 \pm 3.49$ & $71.34 \pm 2.64$ & $81.31 \pm 0.85$ \\
    1 & $64.82 \pm 3.39$ & $15.05 \pm 1.66$ & $0.57 \pm 0.02$ & $64.85 \pm 3.39$ & $65.28 \pm 3.4 $& $71.66 \pm 4.77$ & $81.45 \pm 0.85$ \\
    \end{tabular}
    \caption{Mean and standard deviations for evaluations of the excess risk bound in Lemma~\ref{lemma:risk} at preconditioners obtained by minimizing the excess risk bound ("Exact"), upper-bounds on this bound using the 2-operator and Frobenius norms ("Operator" and "Frobenius") and the inverse covariance of the data-generating distribution ("Sigma"). FM denotes that we optimize over a full-matrix preconditioner (default is diagonal matrix). Directly minimizing the bound from Lemma~\ref{lemma:risk} is highly effective and leads to near-optimal risk for the solution found by PGD, while the operator and Frobenius learning rules do not improve upon the identity matrix. }
    \label{table:optimized_p_risk_bounds}
\end{table}

\subsection{Experimental Details}

In this section we give additional details for the regression experiments presented in the main paper and the additional results shown in Appendix~\ref{app:regression_additional_results}.

\subsubsection{Datasets}
\label{sec:synth_reg_data}
\textbf{Synthetic Regression Datasets}: We generate synthetic regression problems by first sampling a normalized ground-truth weight vector \( \mathbf{w}^* = \mathbf{v} / \norm{\mathbf{v}}, \) where \( \mathbf{v} \sim \mathcal{N}(0,1) \). We then sample features from a diagonal, mean-zero Gaussian distribution and computing targets as the inner product with the ground-truth vector:
\[ \mathbf{x}_i \sim \mathcal{N}(0,\mathbf{\Sigma}), \quad y_i = \mathbf{x}_i^\top \mathbf{w}^* + \epsilon, \] 
where \( \epsilon \sim \mathcal{N}(0,\sigma^2) \) is the target noise. To control the hardness of the optimization problem, we generate ill-conditioned, positive-definite covariances by perturbing the identity matrix with squared Gaussian noise. In particular, we compute 
\[ \mathbf{\Sigma} = \mathbf{I} + \text{Diag}(\mathbf{\delta}^2), \quad \quad \mathbf{\delta} \sim \mathcal{N}(0, \zeta^2 \mathbf{I}). \]
The setting $\zeta^2 = 10$ is used in all experiments but Figure~\ref{fig:lemmas_newton_adagrad}, where $\zeta^2 = 1$ is chosen.
We use a training set of 100 observations in all synthetic regression experiments.
For a model $\mathbf{w}$, we report the true risk
\[ \E_{\mathbf{x}, y}\left[ \mathbf{x}^\top \mathbf{w} - y \right] = (\mathbf{w} - \mathbf{w}^*)^\top \mathbf{\Sigma}(\mathbf{w} - \mathbf{w}^*) + \sigma^2, \]
instead of using a test set when using the original data features. 
When using features from a neural tangent kernel, we instead sample a test of 400 examples and use this to evaluate the model performance.
We repeat all experiments ten times with the \emph{same} ground-truth weights $w^*$ and data covariance $\mathbf{\Sigma}$ to control for randomness in the generation of the training and test sets. Figures show the median and inter-quartile range.
All stochastic experiments on synthetic regression use mini-batches of $5$ examples.

\textbf{UCI Datasets}: We use the \texttt{wine} and \texttt{mushroom} datasets from the UCI dataset repository~\cite{Dua:2019}. We use the training and validation splits\footnote{Publicly available at \url{http://persoal.citius.usc.es/manuel.fernandez.delgado/papers/jmlr/}.} created by Fern{\'a}ndez-Delgado et al.~\cite{fernandez2014we} and used by Arora et al.~\cite{arora2019harnessing}.
In all regression experiments, we randomly subset fifty examples from the training set to fit our models and evaluate on the full validation split. All experiments are repeated ten times to control for the effects of sub-setting the training set. Figures show the median and inter-quartile range.
Stochastic experiments use mini-batches of two examples.

\textbf{Computing Neural Tangent Kernels}: We standardize the data before computing the neural tangent kernel. Neural networks are initialized with standard normal weights and use the so-called NTK parameterization~\cite{park2019effect, jacot2018neural} as well as sigmoid activations. We use the BackPACK library~\cite{Dangel2020BackPACK} to compute the Jacobian of the network output.

\subsubsection{Regression with Neural Tangent Kernels}

Here we provide specific details for the experiments shown in Figures~\ref{fig:ntk_synthetic_regression}, \ref{fig:ntk_mushroom_wine}, and \ref{fig:ntk_synthetic_deterministic}. For Figures~\ref{fig:ntk_synthetic_regression} and \ref{fig:ntk_synthetic_deterministic}, we generate a training set of 100 examples with dimension $d = 20$ as described above. We use minibatches of $5$ examples in the stochastic case.
Figure~\ref{fig:ntk_mushroom_wine} considers the \texttt{mushroom} and \texttt{wine} datasets, where we use training sets of size $50$ and mini-batches of $2$ examples. We experiment with neural tangent kernels for single-layer feed-forward neural networks of width $50$ and $100$. For all three datasets, we plot the step-size for SGD which maximized the convergence rate on the training loss while still converging on all ten repeats. The grid-search uses the following grid:
\[ \eta \in \{20, 10, 5, 2.5, 0.1, 0.5, 0.25, 0.1\}. \]
The final step-size chosen for all three datasets was $\eta = 2.5$.
We use the same step-size grid for Adagrad and plot all step-sizes which do not diverge or under-fit.

\subsubsection{Verification of Lemma~\ref{lemma:unique-span-gen}}
\label{app:unique_span_details}
Here we provide specific details for the experiment shown in Figure~\ref{fig:unique-span-gen}.
We generate a synthetic regression dataset as described above and
sample three random, diagonal positive-definite preconditioners as \( \mxp = \text{diag}(\textbf{v}^2) \), where \( \textbf{v} \sim \mathcal{N}(0, \mathbf{I}) \). We select step-sizes individually for each restart by grid-search over the set 
\begin{align}
   \eta \in \{ 0.1, 0.01, 0.005, 0.001, 5 \times 10^{-4}, 1 \times 10^{-4}, 5 \times 10^{-5}, 1 \times 10^{-5}, 5 \times 10^{-6}, 1 \times 10^{-6} \}. \label{eq:step-size-grid}
\end{align}
Each step-size is picked to minimize the average of the training loss halfway through and at the at the end of training, subject to the constraint that the optimizer does not diverge initially (and then recovers) or at the end of optimization.

\subsubsection{Verification of Lemmas~\ref{lemma:newton} and \ref{lemma:adagrad}}
Here we provide specific details for Figure~\ref{fig:lemmas_newton_adagrad}, which verifies that Newton's method and full-matrix Adagrad converge to the min-norm solution for over-parameterized linear regression problems.
Unlike the other synthetic regression experiments, here we generate a well-conditioned problem ($\zeta^2 = 1$; see above) to avoid complications when computing the Newton step. As in other experiments, we choose step-sizes independently for each repeat (i.e. randomly sampled training set) by grid-search of the following set
\[ \eta \in \{1, 0.1, 0.01, 0.005, 0.001, 5 \times 10^{-4}, 1 \times 10^{-4}, 5 \times 10^{-5}, 1 \times 10^{-5}, 5 \times 10^{-6}, 1 \times 10^{-6} \}. \]
We use the same rule to select step-sizes as described in Appendix~\ref{app:unique_span_details}.

\subsubsection{Verification of Lemma~\ref{lemma:opt-P-reduction}}
Here we provide specific details for Figure~\ref{fig:opt-P-reduction}, which verifies that each interpolating solution has a corresponding preconditioner $P$ such that PGD with this preconditioner converges to the same solution.
We generate a synthetic regression dataset as described above and consider the interpolating solutions found by three optimizers: Adagrad, Adam, and Coin Betting. 
We repeat the step-size selection procedure given in Appendix~\ref{app:unique_span_details} to select step-sizes for these optimizers and then run them until they have converged to interpolating solutions $\wopt$.
We then compute the corresponding preconditioners as described in Lemma~\ref{lemma:opt-P-reduction}.
However, rather than sampling a random vector in the data span, we use the setting
\[ \mathbf{\nu} = \frac{\norm{\wopt}}{\norm{\mathbf{X}^\top \mathbf{y}}} \mathbf{X}^\top \mathbf{y}, \]
which can be interpreted as a normalized, one-step approximation to the min-norm solution.
In practice, this leads to well-conditioned preconditioners, unlike naive random sampling.
We repeat the step-size grid-search to select step-sizes for the PGD optimizers using these preconditioners.

\subsubsection{Verification of Proposition~\ref{prop:proj-span}}
Figure~\ref{fig:proj-span} considers improving the generalization of popular optimizers like Adam, Adagrad, and Coin Betting, which do not converge to the min-norm solution, by projecting their iterations onto $\text{span} (\textbf{X}^\top)$ after each iteration.
Once again, we generate a synthetic regression dataset and use a grid-search to select step-sizes for each algorithm independently on each sampled training set. 
We use the step-size grid in Equation~\ref{eq:step-size-grid}.
The same step-size is used for the projected and unprojected variants of all optimizers. 
Projections onto the data span do not change the optimization dynamics for linear models, since any model components lost in the projection are orthogonal to the gradient.

\subsubsection{Improved Generalization via Better Preconditioners}

We optimize the excess-risk bound and both upper-bounds (\ref{eq:operator_risk_bound} and \ref{eq:frobenius_risk_bound}) on the excess-risk using gradient descent with a fixed step-size. We select the step-size independently for each randomly generated train/test split by grid search. The grid considered is 
\[ \eta \in \{ 5 \times 10^{-1}, 10^{-1}, 10^{-2}, 5 \times 10^{-3}, 10^{-3}, 10^{-4} \}, \]
and final step-sizes are chosen to minimize the value of the optimized risk bound.
We run the optimization procedure for $7500$ iterations or until the gradient norm is smaller than $1 \times 10^{-7}$.
Note that we optimize the risk bound using the combined train/test data, as this does not use knowledge of the test labels. This can be viewed as a form of unsupervised learning, where information from the unlabaled test examples, data covariance $\mathbf{\Sigma}$, and target variance $\sigma^2$ are leveraged to obtain a better preconditioner.
We select step-sizes for PGD with the optimized preconditioners using a search over the grid
\[ \eta \in \{10^{-1}, 10^{-2}, 5 \times 10^{-3}, 10^{-3}, 5 \times 10^{-4}, 10^{-4}, 5 \times 10^{-5}, 10^{-5}, 5 \times 10^{-6}, 10^{-6}, 5 \times 10^{-7}, 10^{-7}\}. \]
We choose final step-sizes as described above in Appendix~\ref{app:unique_span_details}. In practice, we found that the preconditioners obtained by optimizing the exact bound on the excess risk were highly ill-conditioned or indefinite and so required very small step-sizes. Results are provided only for the stochastic setting, as we found the deterministic setting to be virtually identical. We use mini-batches of five examples for all optimizers.
\section{Experiments for linear classification}
\label{app:exp-classification}


This section presents further experimental results for under and over-parameterized linear classification problems. In Appendix~\ref{app:equiv_prec_cls}, we verify the construction of the equivalent preconditioner in Lemma~\ref{lemma:exp-equivalent-precond}. In Appendix~\ref{app:sqh_ball_cls}, we verify that when minimizing the the squared-hinge loss for datasets with a known margin, projections onto the data-span and the $\ell_2$ ball ensure convergence to the max-margin solution. Finally, in Appendix~\ref{app:relative-margin} we provide experimental details for preconditioned gradient descent converging to the  maximum relative margin solution in Figure~\ref{fig:sigma_prec_gd}, and in Appendix~\ref{app:overp_cls}, we presents additional results for over-parameterized linear classification with an exponential-tailed loss.

\subsection{Verification of Lemma \ref{lemma:exp-equivalent-precond}}
\label{app:equiv_prec_cls}

In this section we empirically validate the result of 
Lemma~\ref{lemma:exp-equivalent-precond}: every interpolating solution $\wopt$ with zero training error has a corresponding preconditioner $\mxp$ for which PGD with $\mxp$ converges to $\wopt$. We run the Adam, Adagrad, and Coin Betting optimizers initialized at the origin, and then construct the corresponding preconditioners. We use the exponential loss on a synthetic dataset of 500 training points using 50 features. Figure~\ref{fig:app_cls_equiv_precond} shows the optimization performance of both the original optimizers and their associated PGD methods (dashed lines). The PGD methods converge to models with full training accuracy and the same direction as the original optimizers (right-most plot). Notice that some of the original optimizers (and their preconditioned equivalent) converge to solutions which do not align with the max $\ell_2$ margin solution. 

\begin{figure}[!h]
    \centering
    \includegraphics[width=0.9\textwidth]{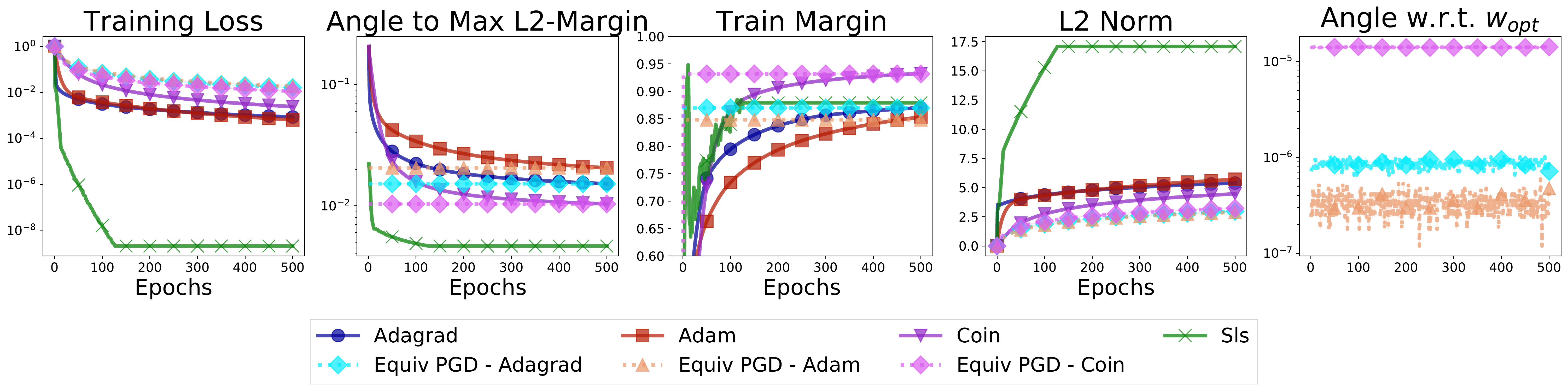}
    \caption{Experimental validation of Lemma~\ref{lemma:exp-equivalent-precond}. For each optimizer, the corresponding dashed line shows the convergence of PGD with a preconditioner constructed from the interpolating solution found by that optimizer as in Lemma~\ref{lemma:exp-equivalent-precond}. The three PGD methods converge in \textit{direction} to the same solutions as the original optimizers, as predicted.}
    \label{fig:app_cls_equiv_precond}
\end{figure}

\subsection{$\ell_2$ ball projection for the squared hinge loss}
\label{app:sqh_ball_cls}

In Section~\ref{sec:experiments}, we explored the effect of projecting onto an appropriate space in order to improve the generalization performance of a linear classifier under exponential-tailed losses. In this section we explore an analogous technique for the case of the squared hinge loss: projecting the iterates of an optimizer onto the $\ell_2$ ball of radius $1/\gamma$, where the margin $\gamma$ is assumed to be known, implies convergence to the $\ell_2$ max-margin solution.

We construct synthetic classification datasets in which the samples from each class follow Gaussian distribution with different parameters for each class. We ensure that the training set is linearly separable and compute the maximum $\ell_2$ margin attainable on it in order to perform the projection. We use a training set of 500 points and explore under- and over-parameterized settings with 100 and 1000-dimensional features, respectively.

In Figure~\ref{fig:app_sqh_ball} we present several learning rate configurations for each optimizer, as well as the effect that the projection onto the $\ell_2$ ball (and data span for the over-parameterized case) have in relation to the convergence to the $\ell_2$ max-margin solution. 

We observe that even though the projection might slow down the convergence in terms of the training loss for large step-sizes, it improves the speed for the accuracy on the training set. Moreover, for every optimizer and step-size configuration, the projection onto the ball improves the convergence towards the max-margin solution compared to the un-projected optimizer. In the right-most plots, all the lines corresponding to projected optimizers overlap at a norm of $1/\gamma$. 
\begin{figure}[!h]
    \centering
    \includegraphics[width=0.9\textwidth]{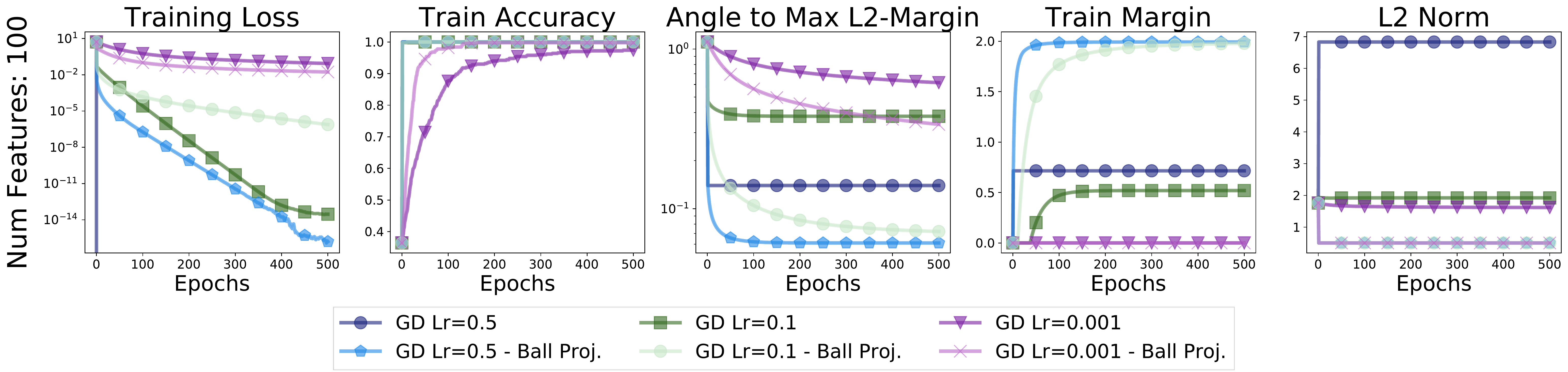}
    \includegraphics[width=0.9\textwidth]{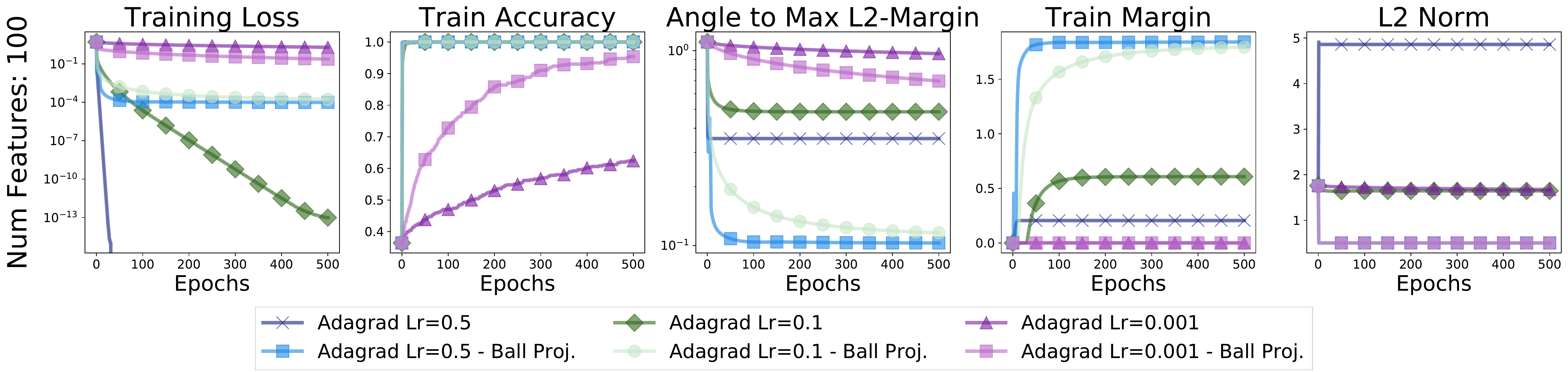}
    \includegraphics[width=0.9\textwidth]{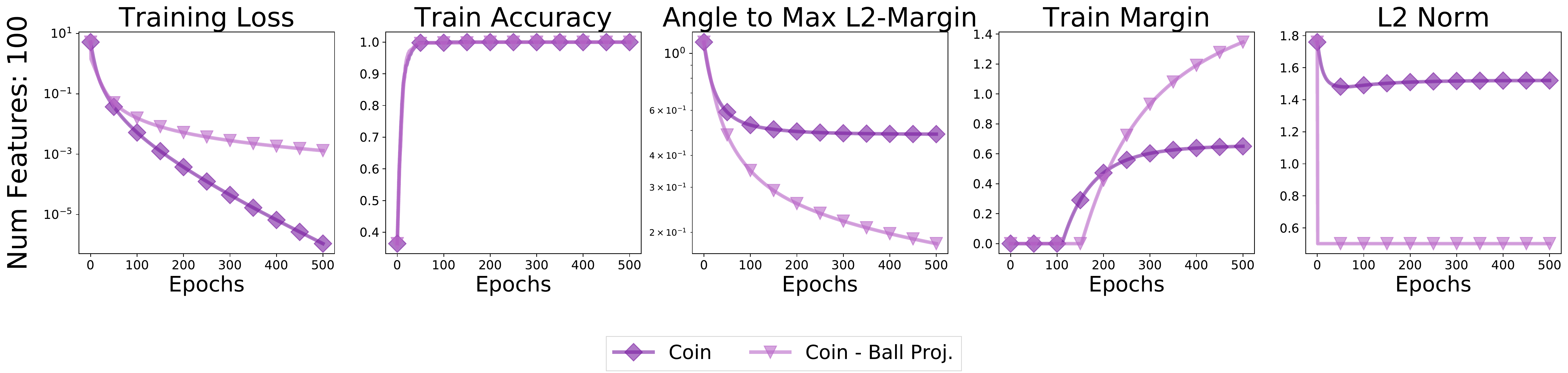}
    \caption{Projected and unprojected variants of several optimizers onto the $\ell_2$ ball of (assumed known) radius $1/\gamma$. }
    \label{fig:app_sqh_ball}
\end{figure}

Finally, in the over-parameterized setting, in Figure~\ref{fig:app_project_sqh_ball}, projection onto the span of the data and subsequent ball projection provides further improvements in terms of angle to the max-margin solution and the margin measured on the training points. 
\begin{figure}
    \centering
    \includegraphics[width=0.9\textwidth]{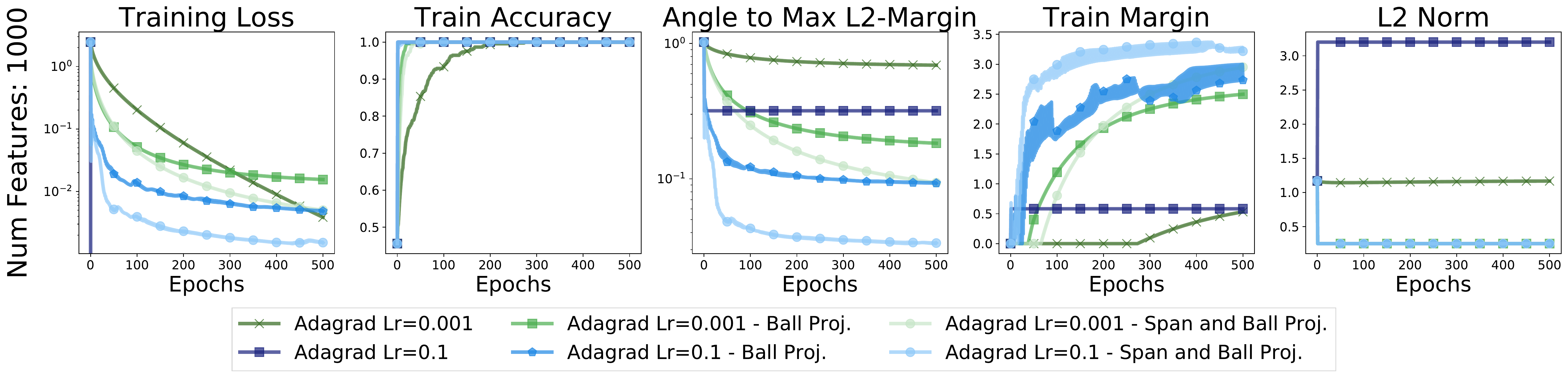}
    \caption{Projected and unprojected variants of Adagrad optimizers onto the data-span and $\ell_2$ ball of (assumed known) radius $1/\gamma$.}
    \label{fig:app_project_sqh_ball}
\end{figure}

\subsection{Maximum relative margin}
\label{app:relative-margin}

We replicate the synthetic setting of of \cite{shivaswamy2010maximum} to study the convergence of preconditioned gradient descent to the maximum relative margin solution. We make minor modifications to the parameters of the originally proposed distribution to allow for separability under a homogeneous linear model. We consider a 2-dimensional classification dataset in which the distributions of each of the classes are Gaussian with matching covariance but different means.
\[ \mathbf{x} \,|\, y=1 \sim \mathcal{N}(\mathbf{\mu_+},\mathbf{\Sigma}) \hspace{1cm} \mathbf{x} \,|\, y=-1 \sim \mathcal{N}(\mathbf{\mu_{-}},\mathbf{\Sigma}) \hspace{1cm} \mathbf{\mu_+} =  -\mathbf{\mu_-} = -\frac{2}{5} [[19, 13]]^\top \hspace{1cm} \mathbf{\Sigma} = \frac{1}{4} \begin{bmatrix}
    17 & 16.9 \\
    16.9 & 17
  \end{bmatrix} \] 

\begin{figure}[!h]
  \begin{center}
    \includegraphics[trim=0 0 0 10, clip, width=0.3\textwidth]{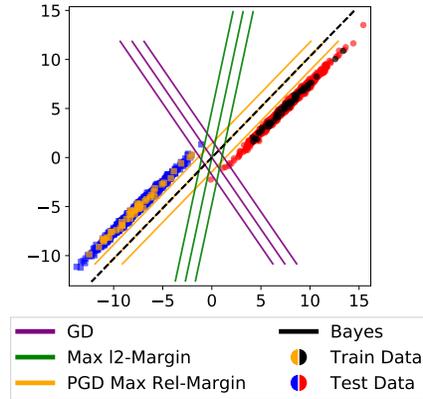}
  \end{center}
  \vspace{-3.5ex}
  \caption{Comparing solutions with maximum $\ell_2$ and relative margin performance on a synthetic mixture of Gaussians dataset. Incorporating the covariance of the data in the form of a preconditioning of the optimization, maximizes the relative margin and results in a solution which is better aligned with the Bayes optimal classifier. Note how, even though the maximum margin classifier is reasonable for the training set, it is ``agnostic'' to directions in which the data tends to spread.}
  \label{fig:sigma_prec_gd_app}
\end{figure}

Figure \ref{fig:sigma_prec_gd_app} displays again the training and test datasets as well as the classifiers presented in the main paper. The training set contains 100 points (orange and black), while the test set is formed of 1900 instances (blue and red). Training statistics for preconditioned gradient descent under the exponential loss are presented in Figure~\ref{fig:app_sigma_prec_training}. The solution found by gradient descent stagnates after a few iterations due to the nature of the loss. However, the misalignment of the gradient descent solution causes it to mis-classify one of the test points. This, coupled with the norm of the iterates of gradient descent, induces a large test loss. Note that the mis-classified test point which represents an outlier for the GD and max $\ell_2$-margin solutions appears precisely in the directions ``corrected'' by the preconditioning via the covariance matrix.

\begin{figure}[!h]
\centering
\includegraphics[width=1\textwidth]{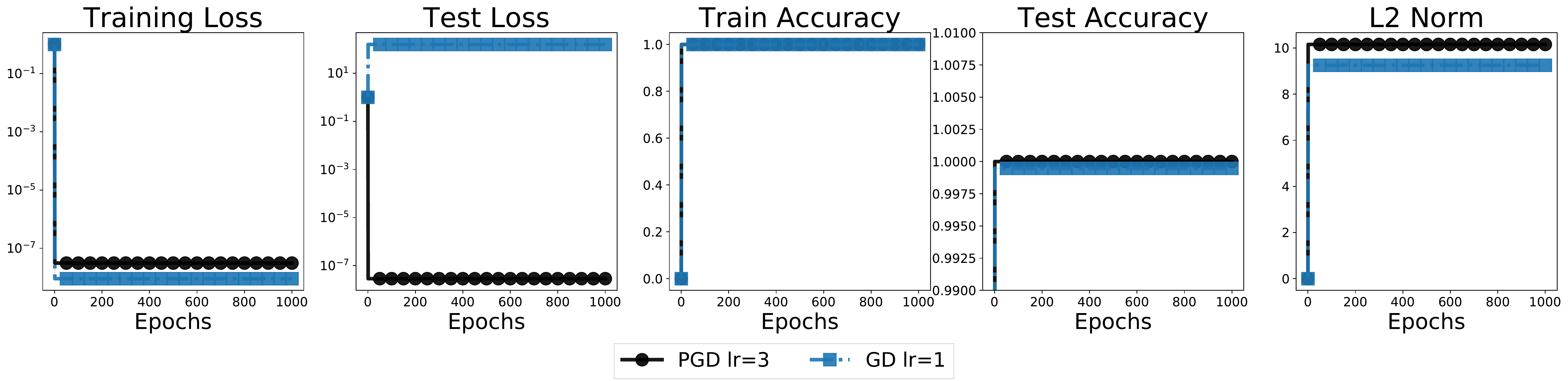}
\caption{Performance of gradient descent with preconditioner $\mathbf{\Sigma}^{-1}$ on a synthetic experiment with Gaussian class-conditionals.}
\label{fig:app_sigma_prec_training}
\end{figure}

\subsection{Over-parameterized linear classification}
\label{app:overp_cls}

The features and labels for this experiment are generated as the synthetic regression dataset presented in Appendix \ref{sec:synth_reg_data} using 1-dimensional targets. The targets are then binarized depending on their sign. We explore different levels of over-parametrization by sampling 300 training points and 600 test points with 500 and 1000 random Gaussian features. 

Figure~\ref{fig:app_overp_classification} presents the impact of projection onto the data span, as well as switching to gradient descent with a step-size $\eta = 10$ (found to perform well via a grid-search). In the legend, the key-word ``Always Project'' indicates that the weights are projected onto the data span at every iteration, while ``Project at Switch'' indicates that only one projection onto the data span is performed and it takes place at the moment at which we switch from Adagrad to gradient descent. We experimented switching at different points in training (50\%, 75\% and 90\%) and obtained qualitatively similar performance as those results presented here. The results are aggregated over 10 data samples and all instances are initialized at the origin.

As expected, we see that the projection results in a smaller $\ell_2$ norm and angle with respect to the $\ell_2$ maximum margin solution. More importantly, these results demonstrate that the generalization performance can be impacted by the ``choice'' of subspace in which the iterates of the optimizer lie.

\begin{figure}[!h]
\centering
\includegraphics[width=0.99\textwidth]{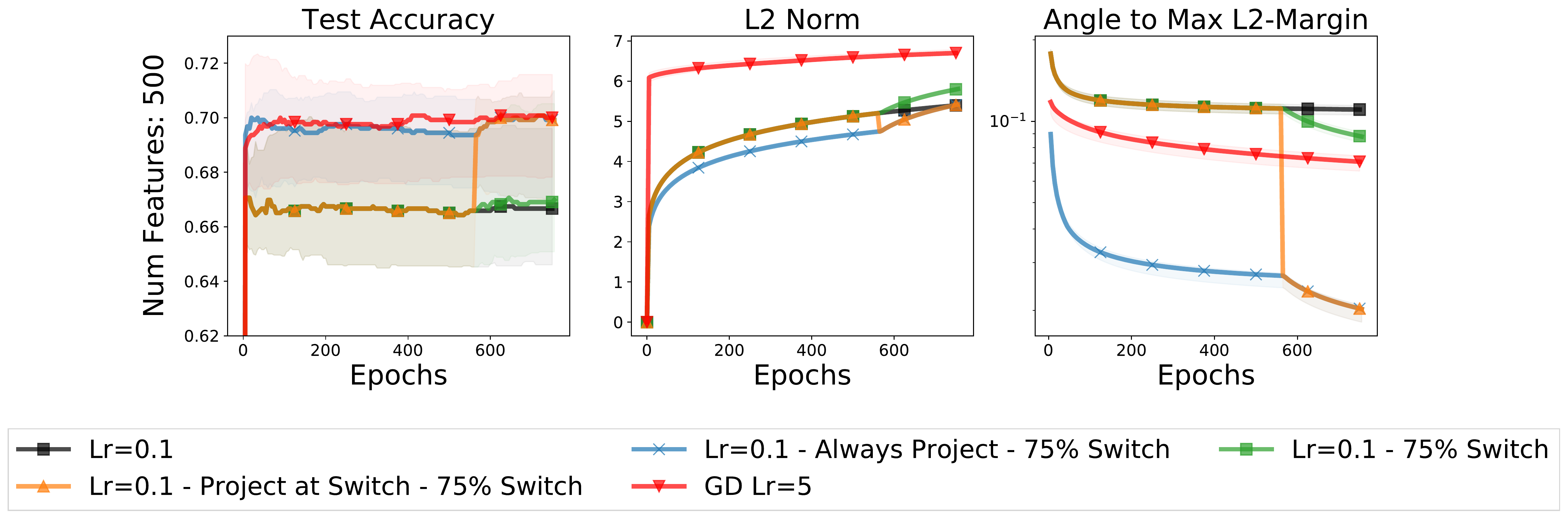}
\includegraphics[width=0.99\textwidth]{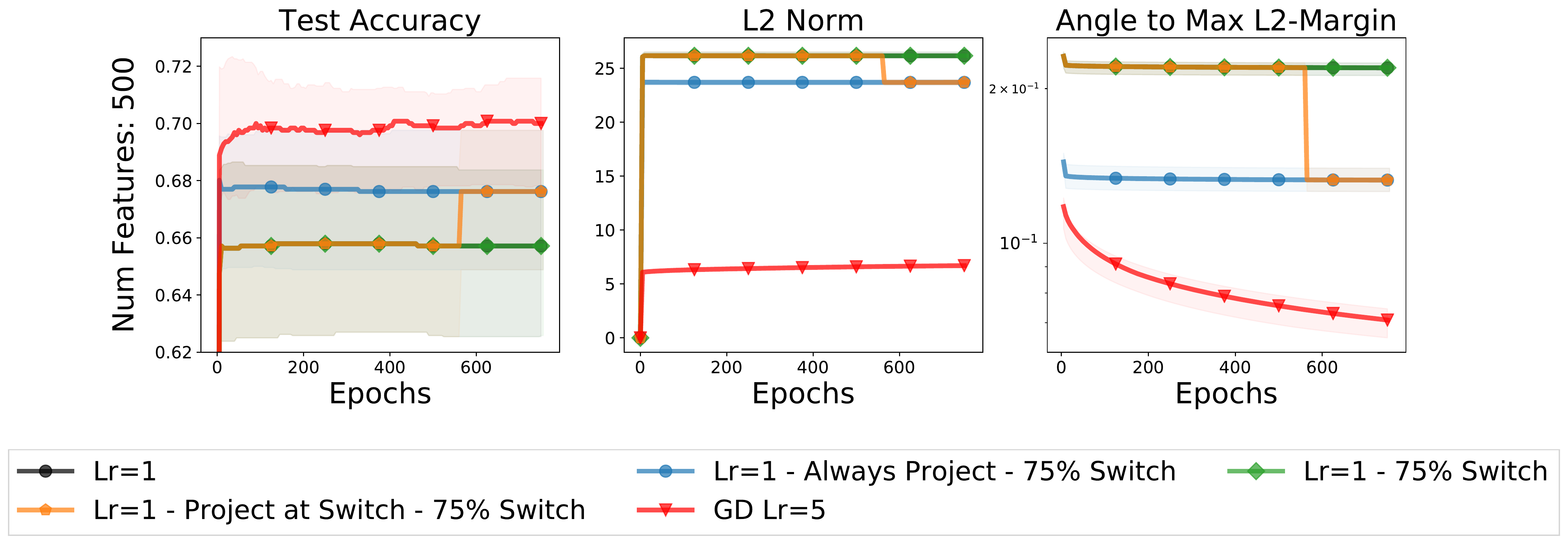}
\includegraphics[width=0.99\textwidth]{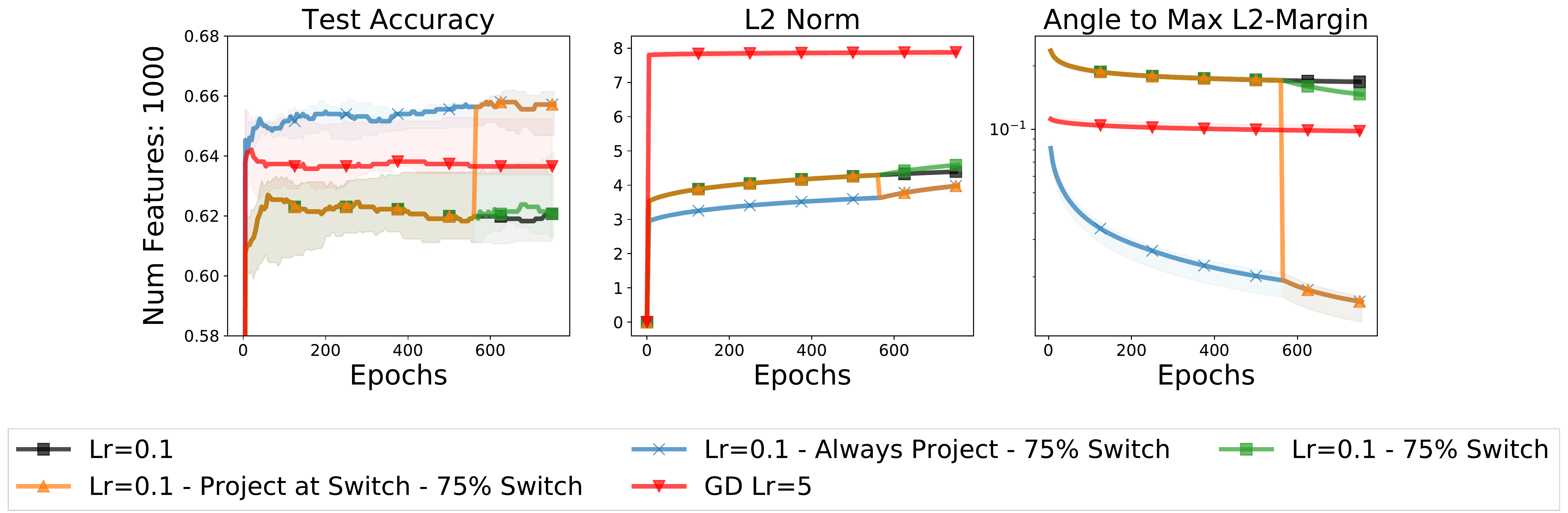}
\caption{Performance of GD and Adagrad on a synthetic overparameterized classification problem with random Gaussian features. Projecting onto the data-span improves the test accuracy, while decreasing the solution's norm and angle to the max-margin solution. 
}
\label{fig:app_overp_classification}
\end{figure}

\end{document}